\documentclass{tlp}

\usepackage{times}
\usepackage{soul}
\usepackage{url}
\usepackage[hidelinks]{hyperref}
\usepackage[utf8]{inputenc}
\usepackage[small]{caption}
\usepackage{graphicx}
\usepackage{amsmath}
\usepackage{booktabs}

\usepackage{algorithm}
\usepackage[noend]{algpseudocode}
\usepackage{enumerate}

\usepackage{tikz}
\usetikzlibrary{decorations.pathmorphing,arrows.meta}
\usetikzlibrary{shapes,decorations}
\tikzstyle{node}=[fill=white, draw=black, shape=rectangle,minimum width=0.6cm, minimum height=0.6cm, node distance=0.6cm]
\tikzstyle{redcirc}=[fill=white, draw=red, shape=circle]
\tikzstyle{bluebox}=[fill=white, draw=blue, shape=rectangle]
\tikzstyle{greenbox}=[fill=white, draw=black!50!green, shape=rectangle]
\pgfdeclarelayer{bg}    % declare background layer
\pgfsetlayers{bg,main}  % set the order of the layers (main is the standard layer)

\usepackage{amsmath}
\usepackage{amssymb}
\usepackage{mathtools}
\usepackage{bbm}
\usepackage{bm}
\usepackage{nicefrac}
\usepackage[noabbrev]{cleveref} %refer to multiple things (eg. equations at once)

\newtheorem{theorem}{Theorem}
\newtheorem{lemma}[theorem]{Lemma}
\newtheorem{define}[theorem]{Definition} 

\newtheorem{example}{Example}

\usepackage{stmaryrd} %brackets
\usepackage{latexsym} %box, diamond

\usepackage{todonotes}
\usepackage{xspace}
\usepackage{color,soul}
\usepackage{verbatim}
\usepackage[switch]{lineno}

\usepackage{wrapfig}

\newcommand{\srzero}{\ensuremath{e_{\oplus}}}
\newcommand{\srone}{\ensuremath{e_{\otimes}}}
\newcommand{\srsplus}{\ensuremath{{\oplus}}}
\newcommand{\srstimes}{\ensuremath{{\otimes}}}
\newcommand{\srbplus}{\ensuremath{{\textstyle\bigoplus}}}
\newcommand{\srbtimes}{\ensuremath{{\textstyle\bigotimes}}}
 
\newcommand{\NP}{\textsc{NP}\xspace}

\newcommand{\PRIM}{\ensuremath{\textsc{PRIM}}}

\newcommand{\smp}{\textsc{sm}ProbLog\xspace}
\newcommand{\dtp}{DTProbLog\xspace}
\newcommand{\SLASH}{SLASH\xspace}
\newcommand{\scp}{SC-ProbLog\xspace}
\newcommand{\twoAMC}{\textsc{2AMC}\xspace}
\newcommand{\aspmc}{{\small\textsf{aspmc}}\xspace}
\newcommand{\problog}{ProbLog\xspace}

\newcommand{\PITA}{PITA\xspace}

\newcommand{\transform}{t}
\newcommand{\inter}{\ensuremath{int}}
\newcommand{\lit}{\ensuremath{lit}}
\newcommand{\mods}{\ensuremath{mod}}
\newcommand{\lpnot}{\ensuremath{not\,}}
\newcommand{\dd}{\!::\!}
\begin{document}

\lefttitle{R. Kiesel \and P. Totis \and A. Kimmig}

\jnlPage{X}{X}
\jnlDoiYr{2022}
\doival{10.1017/xxxxx}

\title[Theory and Practice of Logic Programming]{Efficient Knowledge Compilation Beyond Weighted Model Counting\footnote{Full proofs for all technical results are in the
  technical~\ref{sec:appendix}.}}%\thanks{TODO}}
%\title{Efficient knowledge compilation beyond NP (by leveraging definability)} 
\begin{authgrp}
\author{\sn{Rafael} \gn{Kiesel\textsuperscript{*},} \sn{Pietro} \gn{Totis\textsuperscript{\dag},} \sn{Angelika} \gn{Kimmig\textsuperscript{\dag}}}
\affiliation{\phantom{ }\textsuperscript{*}TU Vienna; \phantom{ }\textsuperscript{\dag}KU Leuven}
\end{authgrp}

\history{\sub{xx xx xxxx;} \rev{xx xx xxxx;} \acc{xx xx xxxx}}

\maketitle
\begin{abstract}
Quantitative extensions of logic programming  often require the solution of so called \emph{second level} inference tasks, i.e., problems that involve a third operation, such as maximization or normalization, on top of addition and multiplication, and thus go beyond the well-known weighted or algebraic model counting setting of probabilistic logic programming under the distribution semantics. We introduce Second Level Algebraic Model Counting (2AMC) as a generic framework for this kind of problems. As 2AMC is to (algebraic) model counting what forall-exists-SAT is to propositional satisfiability, it is notoriously hard to solve. First level techniques based on Knowledge Compilation (KC) have been adapted for specific 2AMC instances by imposing variable order constraints on the resulting circuit. However, those constraints can severely increase the circuit size and thus decrease the efficiency of such approaches. We show that we can exploit the logical structure of a 2AMC problem to omit parts of these constraints, thus limiting the negative effect. Furthermore, we introduce and implement a strategy to generate a sufficient set of constraints statically, with a priori guarantees for the performance of KC. Our empirical evaluation on several benchmarks and tasks confirms that our theoretical results can translate into more efficient solving in practice. Under consideration for acceptance in TPLP.
\end{abstract}

\section{Introduction}
%\ak{fyi, I'll only look at the intro again once the rest has stabilized}
Especially in recent years,  research on quantitative extensions in Logic Programming has flourished. \problog~\citep{de2007problog}, \PITA~\citep{riguzzi2011pita}, \smp~\citep{totis2021smproblog} and others~\citep{baral2009probabilistic, lee2017lpmln} allow for probabilistic reasoning, DeepProbLog~\citep{manhaeve2018deepproblog}, NeurASP~\citep{yang2020neurasp} and SLASH~\citep{skryagin2021slash} additionally integrate neural networks into programs. 

Algebraic Prolog~\citep{kimmig2011algebraic}, algebraic model counting (AMC)~\citep{kimmig2017algebraic} and algebraic answer set counting~\citep{eiter2021treewidth} define  general frameworks based on semirings to express and solve  quantitative  \emph{first level} problems, which compute \emph{one} aggregate over all models, e.g., counting the number of models, or summing or maximizing values associated with them, such as probabilities or utilities. 
%. For example, we in model counting take the \emph{count} of the models, in probabilistic reasoning we take the \emph{sum} of the probabilities of the models and in optimization we take the \emph{maximum (minimum)} of the values associated with a model. 
\cite{kimmig2017algebraic} showed that we can solve first level problems by compiling the logic program   into a \emph{tractable circuit representation}, on which evaluating an AMC task is in polynomial time if the semiring operations have constant cost.%~\cite{kimmig2017algebraic}.

However, many interesting tasks require two kinds of aggregation, and thus are \emph{second level} problems that go beyond AMC. Examples include Maximum A Posteriori (MAP) inference in probabilistic programs, which involves maximizing over some variables while summing over others, 
inference in  \SLASH and \smp, and optimization tasks in decision-theoretic or constrained probabilistic programming languages such as \dtp~\citep{dtproblog,derkinderen2020algebraic} and \scp~\citep{latour2017combining}. %and Maximum A Posteriori (MAP) queries in \problog are \emph{second level} problems. Here, we need to perform \emph{two} aggregation steps. E.g., for MAP, we need to maximize the probability of a program over a set of query atoms. 
%Generally, in such tasks, we first evaluate one aggregate over the \emph{inner} variables to obtain a value for each assignment to the \emph{outer} variables, and then perform a second aggregation on that outer level. 
%the probability for each possible assignment to the query variables and then an aggregate over the \emph{outer} variables to choose the assignment that maximizes this probability.

While second level problems stay hard on general tractable circuit representations, 
%As is, KC is only applicable to second level problems in a limited fashion, since they stay hard on SDDs and d-DNNFs. \emph{Constrained KC (CKC)} alleviates this: 
\dtp and \scp are known to be polynomial time on \emph{$\mathbf{X}$-constrained} SDDs. 
The key idea is to ensure that all variables of the \emph{outer} aggregation appear before those of the \emph{inner} one in the circuit so that we can perform both aggregations sequentially. 
%\rafael{was this to add more meaning to the before? In that sense it does not work, since the outer variables are not smaller than the inner ones according to the partial order induced by the DAG. Otherwise, I think this statement just adds confusion (at least I do not know what it should tell me).
%, which is a directed acyclic graph. 
%The idea behind CKC is that the \emph{outer} variables in $\mathbf{X}$ occur \emph{before} the \emph{inner} variables in the circuit. 
This, however, comes at a high cost: circuits respecting this constraint may be exponentially larger than non-constrained ones, resulting in significantly slower inference. Additionally, certain optimization techniques used in knowledge compilation, such as unit propagation, may cause constraint violations. 
%the constraints on the order in which variables are decided limit our options during compilation severely. Thus, constrained circuits may be exponentially larger than non-constrained ones, resulting in significantly slower inference. Additionally, since the order of the variables is fixed we cannot even decide the literals that are entailed by unit propagation earlier. 
%given the constraints on the order of the variables, computing a good order is an open problem. CKC, as it is in \PITA and \problog, therefore needs to use dynamic reordering of the variables in order to optimize the size of the circuit. \ak{I don't get the last two sentences} \rafael{better?}

In this paper, we generalize the AMC approach to second level problems, and show that we can often weaken the order constraint using \emph{definability}: 
A variable $Y$ is defined by a set of variables $\mathbf{X}$ and a propositional theory $\mathcal{T}$ if the value of $Y$ is functionally determined by the assignment to $\mathbf{X}$ in every satisfying assignment to $\mathcal{T}$. E.g., $a$ is defined by $b$ in the theory $\{a \leftrightarrow b\}$. Informally, if a variable participating in the inner aggregation becomes defined by the variables of the outer aggregation at any point during compilation, we can
%is defined in terms of the variables before it in the order, it becomes possible to 
move that variable to the outer aggregation. This can allow for exponentially smaller circuits and consequently faster evaluation, and additionally justifies the use of unit propagation.

Our main contributions are as follows:
\begin{itemize}
\item We introduce second level algebraic model counting (\twoAMC), a semiring-based unifying framework for second level quantitative  problems, and show that MAP, \dtp and \smp inference are \twoAMC tasks.
%    \item In order to be able to approach quantitative second level problems in a unified framework, we use the versatility of \emph{semirings} to introduce second level algebraic model counting (\twoAMC) and show that \smp, \SLASH, \dtp and MAP inference are special cases of it.
    \item We weaken \emph{$\mathbf{X}$-firstness}, the constraint that the variables in $\mathbf{X}$ need to occur first, to $\mathbf{X}$-firstness \emph{modulo definability} and show that this is sufficient for solving \twoAMC tasks under weak additional restrictions.
    %and demonstrate that it can lead to drastically smaller circuits, while still allowing for the solution of \twoAMC under weak additional restrictions.
    \item We lift methods for generating good variable orders statically from tree decompositions to the constrained setting. %, circumventing the need for dynamic reordering during (C)KC.
    \item We implement our contributions in the algebraic answer set counter \aspmc~\citep{eiter2021treewidth} and the probabilistic reasoning engine ProbLog2~\citep{problog2system}.
%    \rafael{@Pietro: From what I heard \smp is not strictly speaking problog. Can you write something that is correct here?}
 %   to enable efficient constrained compilation modulo definability.
 \item We evaluate our contributions on a range of benchmarks, demonstrating that drastically smaller circuits can be possible, and that our general tools are competitive with state of the art implementations of specific second level tasks in logic programming.
 %   \item Our experimental results show that leveraging definability and statically generated constrained variable orders is not only interesting from a theoretical point of view. Especially enabling unit propagation seems to leads to performance improvements in practice, making our tools competitive with other state of the art implementations for second level problems in logic programs.
\end{itemize}
\section{Preliminaries}\label{sec:preliminaries}
We consider propositional theories $\mathcal{T}$ over variables $\mathbf{X}$. For a set of variables $\mathbf{Y}$, we denote by $\lit(\mathbf{Y})$ the set of literals over $\mathbf{Y}$, by $\inter(\mathbf{Y})$ the set of \emph{assignments} $\mathbf{y}$ to $\mathbf{Y}$ and by $\mods(\mathcal{T})$ those assignments that satisfy $\mathcal{T}$. Here, an assignment is a subset of $\lit(\mathbf{Y})$, which contains exactly one of $a$ and $\neg a$ for each variable $a \in \mathbf{Y}$. Given a partial assignment $\mathbf{y} \in \inter(\mathbf{Y})$ for a theory $\mathcal{T}$ over $\mathbf{X}$, we denote by $\mathcal{T}\mid_{\mathbf{y}}$ the theory over $\mathbf{X}\setminus \mathbf{Y}$ obtained by conditioning $\mathcal{T}$ on $\mathbf{y}$. We use $\models$ for the usual entailment relation of propositional logic.

Algebraic Model Counting (AMC) is a general framework for quantitative reasoning over models that generalizes weighted model counting %\pietro{[at least a reference here?]} 
to the semiring setting~\citep{kimmig2017algebraic}. %\rafael{How about we move the AMC reference to the end and leave it at that, since that covers both AMC and WMC?} % and supports various types of labels/weights.
\begin{define}[Monoid, Semiring]
A \emph{monoid} $\mathcal{M} = (M, \srstimes, \srone)$ consists of an associative binary operation $\srstimes$ with neutral element $\srone$ that operates on elements of $M$.
A commutative \emph{semiring} $\mathcal{S} = (S, \srsplus, \srstimes, \srzero, \srone)$ consists of two commutative monoids $(S, \srsplus, \srzero)$ and $(S, \srstimes, \srone)$ such that $\srstimes$ right and left distributes over $\srsplus$ and $\srzero$ annihilates $S$, i.e. $\forall s \in S: s\srstimes \srzero = \srzero = \srzero\srstimes s$. 
\end{define}
Some examples of well-known commutative semirings are 
\begin{itemize}
\label{ex:semirings}
    \item $\mathcal{P} = ([0,1], +, \cdot, 0, 1)$, the probability semiring,
%    \item $\mathbb{F} = (\mathbb{F}, +, \cdot, 0, 1)$, for $\mathbb{F} \in \{\mathbb{N}, \mathbb{Z}, \mathbb{Q}, \mathbb{R}\}$ the semiring of the numbers in $\mathbb{F}$ with addition and multiplication,
    \item $\mathcal{S}_{\max, \cdot} = (\mathbb{R}_{\geq 0}, \max, \cdot, 0, 1)$, the max-times semiring,
    \item $\mathcal{S}_{\max, +} = (\mathbb{R}\cup\{-\infty\}, \max, +, -\infty, 0)$, the max-plus semiring,
    \item \textsc{EU} $ = (\{(p, eu) \mid p \in [0,1], eu \in \mathbb{R}\}, +, \otimes, (0,0), (1,0))$, the expected utility semiring, where addition is coordinate-wise and
$
    (a_1, b_1) \otimes (a_2, b_2) = (a_1 \cdot a_2, a_2\cdot b_1 + a_1 \cdot b_2).
$
\end{itemize}
% \rafael{Do we want the examples of semirings here or should we introduce them when we use them? The fact that we have different applications with different semirings is also described in the paragraph below anyway.}
An AMC instance $A = (\mathcal{T}, \mathcal{S}, \alpha)$ consists of a theory $\mathcal{T}$ over variables $\mathbf{X}$, a commutative semiring $\mathcal{S}$ and a labeling function $\alpha: \lit(\mathbf{X}) \rightarrow S$. The value of $A$ is \smallskip \\
% \[ 
% \mathit{AMC}(A)= \srbplus_{M\in \mods(\mathcal{T})}\srbtimes_{l\in M}\alpha(l)
% \]
\smallskip
\centerline{$\mathit{AMC}(A)= \srbplus_{M\in \mods(\mathcal{T})}\srbtimes_{l\in M}\alpha(l)$}
% fixed \pietro{[$\mods(\mathcal{T})$ is not defined yet]}
%Using semirings, we can not only have numerical labels (weights) as the ones used in WMC, but also sets (e.g., to collect relevant variables), Boolean formulae (e.g., to obtain explicit representations of models), polynomials (e.g., for sensitivity analysis in probabilistic models), and many more by using an appropriate semiring. It thus provides a framework that covers many different tasks from a variety of different fields, e.g., probabilistic inference, \#SAT, and weighted model counting.
A prominent example of AMC is inference in probabilistic logic programming, which uses the probability semiring $\mathcal{P}$ and a theory for a probabilistic logic program (PLP). A PLP $\mathcal{L}=\mathbf{F}\cup R$ is a set $\mathbf{F}$ of probabilistic facts $p\dd f$ and a set $R$ of rules $h \leftarrow b_1,\dots,b_{n},  \lpnot c_1, \dots, \lpnot c_{m}$ such that each assignment (world) $\mathbf{f}$ to $\mathbf{F}$ leads to exactly one model.
%The \emph{head} $h$ is an atom and the \emph{body} $b_1,\dots,b_{n}$ is a logical conjunction of literals.  
% Probabilistic facts are expressions of the form $p\dd f$ where $f$ is a fact and $p$ is a probability.
%We use the standard terminology about atoms, terms and literals, and
By abuse of notation, we use $\mathcal{L}$ also for a propositional theory with the same models. We denote the Herbrand base, i.e., the set of all ground atoms in $\mathcal{L}$, by $\mathbf{H}$.
The success probability of a query $q$ is 
% $
% \mathit{SUCC(q)} =
% \sum_{\omega\in\inter(\mathbf{F}), q \in \mathcal{L}_{\omega}} P(\omega)
% $
$\mathit{SUCC(q)} = \sum_{\mathbf{h}\in\mods(\mathcal{L}\mid_{q})} P(\mathbf{h})$, the sum of the probabilities $P(\mathbf{h}) = \prod_{l \in \mathbf{h}} \alpha(l)$ of the models $\mathbf{h}$ where $q$ \emph{succeeds}. Here, $\alpha(f) = p, \alpha(\lnot f) = 1-p$ for $p\dd f \in \mathbf{F}$ and $\alpha(l) = 1$, otherwise.

\begin{example}[Running Example]\label{ex:running}
We use the following probabilistic logic program $\mathcal{L}_{ex}$ throughout the paper.
% \begin{align*}
% 0.4\dd a & & c \leftarrow a & &
% 0.6\dd b & & d \leftarrow b 
% \end{align*}
\centerline{$\hfill 0.4\dd a \hfill c \leftarrow a \hfill 0.6\dd b \hfill d \leftarrow b \hfill$}
Its four worlds are $\mathbf{f}_1=\{a,b\}, \mathbf{f}_2=\{a,\lnot b\}, \mathbf{f}_3=\{\lnot a, b\}, \mathbf{f}_4 = \{\lnot a, \lnot b\}$, and the probabilities of their models $\mathbf{h}_1, \mathbf{h}_2, \mathbf{h}_3, \mathbf{h}_4$ are $P(\mathbf{h}_1) = 0.24, P(\mathbf{h}_2) = 0.16, P(\mathbf{h}_3) = 0.36$ and $P(\mathbf{h}_4) = 0.24$. The query $c$ succeeds in the models $\mathbf{h}_1$ and $\mathbf{h}_2$, thus $SUCC(c) = P(\mathbf{h}_1)+P(\mathbf{h}_2)=0.4$.
%, similarly, $P(d)=P(\mathbf{f}_1)+P(\mathbf{f}_3)=0.6$.
\end{example}

\cite{kimmig2017algebraic} showed that Knowledge Compilation (KC) to sd-DNNFs solves any AMC problem. sd-DNNFs are special negation normal forms (NNFs). An NNF~\citep{Darwiche04} is a rooted directed acyclic graph in which each leaf node is labeled with a literal, true or false, and each internal node is labeled with a conjunction $\wedge$ or disjunction $\vee$. %For sake of simplicity, we assume that each internal node has exactly two children. However, all the results and definitions can easily be generalized to NNFs where this is not the case. 
For any node $n$ in an NNF graph, $\mathit{Vars}(n)$ denotes all variables in the subgraph rooted at $n$. By abuse of notation, we use $n$ also to refer to the formula represented by the graph $n$. sd-DNNFs are NNFs that satisfy the following additional properties:
%\begin{itemize}
%    \item Decomposability (D) holds, when $\mathit{Vars}(n_i) \cap \mathit{Vars}(n_j) = \emptyset$ for any
%two children $n_i$ and $n_j$ of an and-node $n$.
%    \item Determinism (d) holds, when $n_i \wedge n_j$ is logically inconsistent for any two children $n_i$ and $n_j$ of an or-node $n$.
%    \item Smoothness (s) holds, when $\mathit{Vars}(n_i) = \mathit{Vars}(n_j)$ for any two children $n_i$ and $n_j$ of an or-node $n$.
%\end{itemize}
\begin{description}
    \item[Decomposability (D):]  $\mathit{Vars}(n_i) \cap \mathit{Vars}(n_j) = \emptyset$ for any
two children $n_i$ and $n_j$ of an and-node.
    \item[Determinism (d):]  $n_i \wedge n_j$ is logically inconsistent for any two children $n_i$ and $n_j$ of an or-node.
    \item[Smoothness (s):]  $\mathit{Vars}(n_i) = \mathit{Vars}(n_j)$ for any two children $n_i$ and $n_j$ of an or-node.
\end{description}
In order to solve second level problems, we need Constrained KC (CKC), i.e., an additional property on NNFs, apart from \textbf{s}, \textbf{d} and \textbf{D}, that restricts the order in which variables occur. %We adapt the definition of the single mixed path property of~\cite{latour2017combining} to our purposes.
\begin{define}[$\mathbf{X}$-Firstness]
Given an NNF $n$ on variables partitioned into $\mathbf{X}, \mathbf{Y}$, we say an internal node $n_i$ of $n$ is \emph{pure} if $\mathit{Vars}(n_i) \subseteq \mathbf{X}$ or $\mathit{Vars}(n_i) \subseteq \mathbf{Y}$ and \emph{mixed} otherwise. 
$n$ is an \emph{$\mathbf{X}$-first} NNF, if for each of its and-nodes $n_i$ either all children of $n_i$ are pure nodes, or one child is mixed and all other children $n_j$ of $n_i$ are pure with $\mathit{Vars}(n_j) \subseteq \mathbf{X}$. 
\end{define}
Note that $\mathbf{X}$-first NNFs contain a node $n_{\mathbf{x}}$ equivalent to $n\mid_{\mathbf{x}}$ for each $\mathbf{x} \in \inter(\mathbf{X})$.

\begin{example}[cont.]
% The SDDs in Figure~\ref{fig:sdds} are both equivalent to $\mathcal{L}_{ex}$. Here, circles correspond to SDDs of the form $\alpha = \{(p_1, s_1), \dots, (p_n, s_n)\}$ and the number in them correspond to the vtree node that they are normalized for. The boxes below the circles then correspond to the individual pairs $(p_i, s_i)$, where SDDs for literals are denoted by the literal itself.

% The SDDs are normalized for the corresponding vtree in Figure~\ref{fig:vtrees}. Therefore, the left SDD is $\mathbf{X}$-constrained, whereas the right is not.

The NNFs in Figure~\ref{fig:sddnnfs} are both sd-DNNFs and equivalent to $\mathcal{L}_{ex}$. 
The left sd-DNNF is furthermore an $\{a,b\}$-first sd-DNNF, whereas the right is not.

\newcommand{\fac}{0.6}
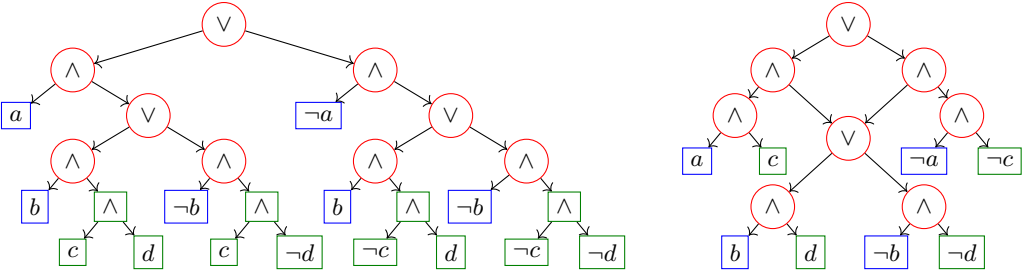
\begin{figure}
    \centering
    \begin{tikzpicture}
        \node[style=redcirc] (O10) at (-0.25, 0) {$\vee$};
        \node[style=redcirc] (A10) at (-2.25, -1*\fac) {$\wedge$};
        \node[style=bluebox] (L10) at (-3, -2*\fac) {$a$};
        \node[style=redcirc] (O20) at (-1.25, -2*\fac) {$\vee$};
        \node[style=redcirc] (A20) at (-2.25, -3*\fac) {$\wedge$};
        \node[style=bluebox] (L20) at (-2.75, -4*\fac) {$b$};
        \node[style=greenbox] (A30) at (-1.75, -4*\fac) {$\wedge$};
        \node[style=greenbox] (L30) at (-2.25, -5*\fac) {$c$};
        \node[style=greenbox] (L31) at (-1.25, -5*\fac) {$d$};
        
        \node[style=redcirc] (A21) at (-0.25, -3*\fac) {$\wedge$};
        \node[style=bluebox] (L21) at (-0.75, -4*\fac) {$\neg b$};
        \node[style=greenbox] (A31) at (0.25, -4*\fac) {$\wedge$};
        \node[style=greenbox] (L32) at (-0.25, -5*\fac) {$c$};
        \node[style=greenbox] (L33) at (0.75, -5*\fac) {$\neg d$};
        
        \node[style=redcirc] (A11) at (1.75, -1*\fac) {$\wedge$};
        \node[style=bluebox] (L11) at (1, -2*\fac) {$\neg a$};
        \node[style=redcirc] (O21) at (2.75, -2*\fac) {$\vee$};
        \node[style=redcirc] (A22) at (1.75, -3*\fac) {$\wedge$};
        \node[style=bluebox] (L22) at (1.25, -4*\fac) {$b$};
        \node[style=greenbox] (A32) at (2.25, -4*\fac) {$\wedge$};
        \node[style=greenbox] (L34) at (1.75, -5*\fac) {$\neg c$};
        \node[style=greenbox] (L35) at (2.75, -5*\fac) {$d$};
        
        \node[style=redcirc] (A23) at (3.75, -3*\fac) {$\wedge$};
        \node[style=bluebox] (L23) at (3, -4*\fac) {$\neg b$};
        \node[style=greenbox] (A33) at (4.25, -4*\fac) {$\wedge$};
        \node[style=greenbox] (L36) at (3.75, -5*\fac) {$\neg c$};
        \node[style=greenbox] (L37) at (4.75, -5*\fac) {$\neg d$};
        
        \draw[->] (O10) to (A10);
        \draw[->] (A10) to (O20);
        \draw[->] (A10) to (L10);
        \draw[->] (O20) to (A20);
        \draw[->] (A20) to (A30);
        \draw[->] (A20) to (L20);
        \draw[->] (A30) to (L30);
        \draw[->] (A30) to (L31);
        \draw[->] (O20) to (A21);
        \draw[->] (A21) to (L21);
        \draw[->] (A21) to (A31);
        \draw[->] (A31) to (L32);
        \draw[->] (A31) to (L33);
        
        \draw[->] (O10) to (A11);
        \draw[->] (A11) to (O21);
        \draw[->] (A11) to (L11);
        \draw[->] (O21) to (A22);
        \draw[->] (A22) to (A32);
        \draw[->] (A22) to (L22);
        \draw[->] (A32) to (L34);
        \draw[->] (A32) to (L35);
        \draw[->] (O21) to (A23);
        \draw[->] (A23) to (A33);
        \draw[->] (A23) to (L23);
        \draw[->] (A33) to (L36);
        \draw[->] (A33) to (L37);
    \end{tikzpicture}
    \hfill
    \begin{tikzpicture}
        \node[style=redcirc] (O10) at (0, 0) {$\vee$};
        \node[style=redcirc] (A10) at (-1, -1*\fac) {$\wedge$};
        \node[style=redcirc] (A20) at (-1.5, -2*\fac) {$\wedge$};
        \node[style=bluebox] (L10) at (-2, -3*\fac) {$a$};
        \node[style=greenbox] (L11) at (-1, -3*\fac) {$c$};
        \node[style=redcirc] (A11) at (1, -1*\fac) {$\wedge$};
        \node[style=redcirc] (A21) at (1.5, -2*\fac) {$\wedge$};
        \node[style=bluebox] (L12) at (1, -3*\fac) {$\neg a$};
        \node[style=greenbox] (L13) at (2, -3*\fac) {$\neg c$};
        
        \node[style=redcirc] (O20) at (0, -2.5*\fac) {$\vee$};
        \node[style=redcirc] (A30) at (-1, -4*\fac) {$\wedge$};
        \node[style=bluebox] (L30) at (-1.5, -5*\fac) {$b$};
        \node[style=greenbox] (L31) at (-0.5, -5*\fac) {$d$};
        \node[style=redcirc] (A31) at (1, -4*\fac) {$\wedge$};
        \node[style=bluebox] (L32) at (0.5, -5*\fac) {$\neg b$};
        \node[style=greenbox] (L33) at (1.5, -5*\fac) {$\neg d$};

        \draw[->] (O10) to (A10);
        \draw[->] (A10) to (O20);
        \draw[->] (A10) to (A20);
        \draw[->] (A20) to (L10);
        \draw[->] (A20) to (L11);
        \draw[->] (O20) to (A30);
        \draw[->] (A30) to (L30);
        \draw[->] (A30) to (L31);
        \draw[->] (O20) to (A31);
        \draw[->] (A31) to (L32);
        \draw[->] (A31) to (L33);
        
        \draw[->] (O10) to (A11);
        \draw[->] (A11) to (A21);
        \draw[->] (A21) to (L12);
        \draw[->] (A21) to (L13);
        \draw[->] (A11) to (O20);
    \end{tikzpicture}
    \caption{Two sd-DNNFs for $\mathcal{L}_{ex}$. Mixed nodes for partition $\{a,b\},\{c,d\}$ are circled red and pure nodes are boxed blue or boxed green, when they are from $\{a,b\}$ or $\{c,d\}$, respectively.}
    \label{fig:sddnnfs}
\end{figure}
\end{example}
\section{Second Level Algebraic Model Counting}
%The semantics of different second level problems are not defined in the same manner. This makes it hard to consider them in a uniform manner. Therefore, we 
We now introduce Second Level Algebraic Model Counting (\twoAMC), a generalization of AMC that provides a unified view on second level problems.
\begin{define}[Second Level Algebraic Model Counting (\twoAMC)]
A \twoAMC instance is a tuple $A = (\mathcal{T}, \mathbf{X}_I, \mathbf{X}_O, \alpha_I, \alpha_O, \mathcal{S}_I, \mathcal{S}_O, \transform)$, where
\begin{itemize}
    \item $\mathcal{T}$ is a propositional theory,
    \item $(\mathbf{X}_I, \mathbf{X}_O)$ is a partition of the variables in $\mathcal{T}$, 
    \item $\mathcal{S}_j = (S_j, \srsplus^{j}, \srstimes^{j}, e_{\srsplus^{j}}, e_{\srstimes^{j}})$ for $j \in\{ I,O\}$ is a commutative semiring,
    \item $\alpha_j : \lit(\mathbf{X}_j) \rightarrow S_j$ 
    %assigns each literal from $\mathbf{X}_j$ a weight in the semiring $\mathcal{R}_j$ 
    for $j \in\{ I, O\}$ is a  labeling function for literals, and 
    % \rafael{I added a definition for $\lit$. I think we need weights for both $a$ and $\neg a$ sometimes, e.g. when the weights are utilities.}
    \item $\transform : S_I \rightarrow S_O$ is a weight transformation function that respects 
    %transforms a weight from $\mathcal{R}_I$ into a weight from $\mathcal{R}_O$ such that 
    $\transform(e_{\srsplus^{I}}) = e_{\srsplus^{O}}$.
\end{itemize}
The \emph{value of $A$}, denoted by $\twoAMC(A)$, is defined as 
\begin{align*}
    \twoAMC(A) = \srbplus^{O}_{\mathbf{x}_O \in \inter(\mathbf{X}_O)} \srbtimes^{O}_{x \in \mathbf{x}_O} \alpha_O(x) \srstimes^{O} \transform\left(\srbplus^{I}_{\mathbf{x}_I \in \mods(\mathcal{T}\mid_{\mathbf{x}_O})} \srbtimes^{I}_{y \in \mathbf{x}_I} \alpha_I(y)\right).
\end{align*}
\end{define}
% \rafael{added a definition of $\inter$, do we want to introduce it (and propositional theories) first, before ProbLog?}
An AMC instance is a \twoAMC instance, where $\mathbf{X}_O = \emptyset$ and $\transform$ is the identity function, meaning we only sum up weights over $\mathcal{S}_I$. Intuitively, the idea behind \twoAMC is that we solve an \emph{inner} AMC instance over the variables in $\mathbf{X}_I$ for each assignment to $\mathbf{X}_O$. Then we apply the transformation function to the result, thus replacing the inner summation by a corresponding element from the outer semiring, and solve a second \emph{outer} AMC instance over the variables in $\mathbf{X}_O$. 
% \ak{we need an example here to make this tangible}
% \rafael{Is the following a good example, or do we want it to be less abstract?}
% \ak{tangible = use running example if possible, and write down eg a max-sum version of the equation with concrete sets of variables; doesn't necessarily need the full formal details of the semirings, but has to help readers catch the idea -- if we want to talk about PMC, then that should go in line with MAP, MEU and SM (but not sure we need/want another one there)}
% \rafael{makes sense.
% TODO: redo this.}
\begin{example}[cont.]
\label{ex:twoamc}
Consider the question whether it is more likely for $c$ to be true or false in $\mathcal{L}_{ex}$ from Example~\ref{ex:running}, i.e., we want to find $\arg\max_{\mathbf{c}\in\inter(c)}\mathit{SUCC(\mathbf{c})}$. To keep notation simple, we consider $\max$ rather than $\arg\max$ (see the discussion below Def.~\ref{def:map} for full details). Denoting the label of literal $l$ by $\alpha(l)$ as in the definition of $\mathit{SUCC}$, the task corresponds to
\[
\max_{\mathbf{c}\in\inter(c)} \alpha(\mathbf{c})\sum_{\mathbf{h}\in\mods(\mathcal{L}_{ex}\mid_{\mathbf{c}})} \prod_{l \in \mathbf{h}}\alpha(l).
\]
% Note that as each world $\mathbf{f}$ of a PLP fixes a unique truth value for each atom, we can label both literals of $c$ and $d$ with $1$ and include those literals in the product without changing the task: 
% $$\max_{\mathbf{c}\in\inter(c)} \alpha(\mathbf{c})\cdot \sum_{\substack{\mathbf{f}\in\inter(\{a,b\}), \mathbf{d}\in\inter(d)\\  R\cup \mathbf{f} \models \mathbf{c}\wedge\mathbf{d}}} \alpha(\mathbf{d})\cdot\prod_{l \in \mathbf{f}}\alpha(l) = \max_{\mathbf{c}\in\inter(c)} \alpha(\mathbf{c})\sum_{\mathbf{x}\in\mods(R \mid_{\mathbf{c}})} \alpha(\mathbf{d})\prod_{l \in \mathbf{x}}\alpha(l)$$
We thus have a \twoAMC task with outer variables $\{c\}$, inner variables $\{a,b,d\}$ and the probability semiring and max-times semiring as inner and outer semiring, respectively, and both kinds of labels given by $\alpha$.  
The formal definition of this \twoAMC instance is $A_{ex} = (\mathcal{L}_{ex}, \{a,b,d\},\{c\}, \alpha_I, \alpha_O, \mathcal{P}, \mathcal{S}_{\max,\cdot}, id)$, where $\alpha_I(l)$ is the probability of $l$ if $l \in \lit(\{a,b\})$ and $1$ if $l \in \lit(\{d\})$, and $\alpha_O(l) = 1$, $l \in \lit(\{c\})$. 
We can further evaluate the value as follows:
\begin{align*}
    \twoAMC(A_{ex})&= \max_{\mathbf{c}\in\inter(c)} 1 \cdot \sum_{\substack{\mathbf{a}\in\inter(\{a\}),\mathbf{b}\in\inter(b), \mathbf{d}\in\inter(d)\\ \mathbf{a}=\mathbf{c}, \mathbf{b}=\mathbf{d}}} 1 \cdot \alpha_I(\mathbf{a})\cdot \alpha_I(\mathbf{b})\\
   & =  \textstyle\max\{\alpha_{I}(a)\alpha_{I}(b) + \alpha_{I}(a)\alpha_{I}(\lnot b), \alpha_{I}(\lnot a)\alpha_{I}(b) + \alpha_{I}(\lnot a)\alpha_{I}(\lnot b)\} \\   
   & =  \textstyle\max\{0.4\cdot 0.6 + 0.4\cdot 0.4,0.6\cdot 0.6 + 0.6\cdot 0.4\} 
   =  \textstyle\max\{0.4, 0.6\} = 0.6
\end{align*}
i.e., the most likely value is $0.6$ and corresponds to $\neg c$. 
\end{example}
Before we further illustrate this formalism with tasks from the literature, we prove that \twoAMC can be solved in polynomial time on $\mathbf{X}_O$-first sd-DNNFs.
A similar result is already known for \dtp~\citep{derkinderen2020algebraic} and \scp~\citep{latour2017combining}. %However, this theorem illustrates the usefulness of \twoAMC as a unifying formalism and will be used as a basis for other Theorems.
\begin{theorem}[Tractable \twoAMC with $\mathbf{X}_O$-first sd-DNNFs]
Let $A = (\mathcal{T}, \mathbf{X}_I, \mathbf{X}_O, \alpha_I, \alpha_O, \mathcal{S}_I, \mathcal{S}_O, \transform)$ be a \twoAMC instance, where $\mathcal{T}$ is an $\mathbf{X}_O$-first sd-DNNF. Then, we can compute $\twoAMC(A)$ in polynomial time in the size of $\mathcal{T}$ assuming constant time semiring operations.
\end{theorem}
\label{thm:tract1}
\begin{proof}[Proof (Sketch)]
% The basic idea is as follows: We can see $\mathcal{T}$ as an algebraic circuit, by replacing or-nodes, and-nodes, false and true by sum, product, zero and one, respectively and by replacing all literals by their weight. Then we only need to make sure to use the sum, product, zero and one from the correct semiring: For pure nodes $n$ such that $\mathit{Vars}(n) \subseteq X_I$ this is the inner one for all other nodes it is the outer one. Additionally, for mixed nodes $n = n_1 \srstimes^O n_2$, where w.l.o.g. $\mathit{Vars}(n_1) \subseteq X_I$  we need to use $\transform(n_1) \srstimes^O n_2$ to have values that are over the same semiring.

Consider a subgraph $n$ of $\mathcal{T}$ with exactly one outgoing edge for each or-node and all outgoing edges for each and-node. As $\mathcal{T}$ is $\mathbf{X}_O$-first and smooth, there is a node $n'$ in $n$ such that $\mathit{Vars}(n') = X_I$, i.e., exactly the outer variables occur above $n'$ (see also the lowest and-nodes of the left NNF in Figure~\ref{fig:sddnnfs}).
%all  outer variables and no inner variables occur above $n'$. 
Thus, $n'$ is equivalent to $\mathcal{T}\mid_{\mathbf{x}_O}$ for some assignment $\mathbf{x}_O$ to the outer variables, for which $n'$ computes the value of the inner AMC instance. 
As  evaluation sums over all these subgraphs, it obtains the correct result.
%\rafael{The important part is the second one, we can probably leave out the first one.}
\end{proof}

We illustrate 2AMC with three tasks from quantitative logic programming. 

\paragraph{\textbf{Maximum a Posteriori Inference}}
A typical second level probabilistic inference task is \emph{maximum a posteriori} inference,
which involves maximizing over one set of variables while summing over another, as in Example~\ref{ex:twoamc}.
%which is the problem of finding the most likely truth-assignment to a set of atoms given some evidence.
\begin{define}[The MAP task]\label{def:map}
\textbf{Given}
a probabilistic logic program $\mathcal{L}$, a conjunction $\mathbf{e}$ of observed  literals for the set of evidence atoms $\mathbf{E}$, and a set of ground query atoms $\mathbf{Q}$ \\
\textbf{Find} 
the most probable assignment $\mathbf{q}$ to $\mathbf{Q}$ given the evidence $\mathbf{e}$, with $\mathbf{R} = \mathbf{H}\setminus (\mathbf{Q}\cup\mathbf{E})$: 
\[\mathit{MAP}(\mathbf{Q}|\mathbf{e})= \arg\max_{\mathbf{q}} P(\mathbf{Q}=\mathbf{q}|\mathbf{e}) = \arg\max_{\mathbf{q}} \sum_{\mathbf{r}}P(\mathbf{Q}=\mathbf{q},\mathbf{e},\mathbf{R}=\mathbf{r} )
\]
\end{define}
\vspace{-12pt}
\subparagraph{\textit{MAP as 2AMC.}} Solving MAP requires (1) summing probabilities over the truth values of the atoms in $\mathbf{R}$ (with fixed truth values for atoms in $\mathbf{E}\cup\mathbf{Q}$), and (2) determining truth values of the atoms in $\mathbf{Q}$ that maximize this inner sum.  
%The two problems in the MAP task are: (1) computing the probability of a truth assignment $\textbf{q}$ and (2) extracting the assignment with the highest probability.
Thus, we have $\mathbf{X}_O=\mathbf{Q}$ and $\mathbf{X}_I=\mathbf{E}\cup\mathbf{R}$. 

%\rafael{Do we want to get rid of the part that keeps track of the assignment again and only say something like "This can also be adapted to keep track of the assignment but is left out here for space reasons."?}
The inner problem corresponds to usual probabilistic inference, i.e., $\mathcal{S}_I=\mathcal{P}$ and $\alpha_I$ assigns $1$ to the literals in $\mathbf{e}$ and $0$ to their negations, $p$ and $1-p$ to the positive and negative literals for probabilistic facts $p\dd f$ that are not part of $\mathbf{E}$, and $1$ to both literals for all other variables in $\mathbf{X}_I$. 
%Therefore, Problem 1 is solved by the probability semiring $\mathcal{R}_I=\mathcal{P}$, and the function $\alpha_I$ which assigns a literal $l \in \lit(\mathbf{X}_I)$ $p$ if $l$ is a true probabilistic fact and $1-p$ if it is a false probabilistic fact. 
%If $l$ is in the evidence, then we assign it $1$ and its negation $0$. 
%All other, non-probabilistic and non-evidence literals have weight $1$ regardless of their truth value. 
%$\mathbf{X}_I$ thus includes all the atoms that are not in $\mathbf{Q}$.
%$\mathbf{X}_O$ is the set of query atoms $\mathbf{Q}$. 

We choose $\mathcal{S}_O = (\mathbb{R}_{\geq 0} \times 2^{\lit(\mathbf{Q})}, \srsplus, \srstimes, (0, \emptyset), (1, \emptyset))$ as the max-times semiring combined with the subsets of the query literals $\lit(\mathbf{Q})$ to remember the assignment that was used. 
Here, $(r_1, S_1) \srsplus (r_2, S_2)$ is $(r_1, S_1)$ if $r_1 > r_2$ and $(r_2, S_2)$ if $r_1 < r_2$. To ensure commutativity, if $r_1=r_2$, the sum is $(r_1, \min_{>}(S_1,S_2))$ where $>$ is some arbitrary but fixed total order on $2^{\lit(\mathbf{Q})}$.
%Given two elements $(r_1, S_1), (r_2, S_2)$, their sum $(r_1, S_1) \srsplus (r_2, S_2)$ is $(r_1, S_1)$ if $r_1$ is greater than $r_2$ and $(r_2, S_2)$ if $r_2$ is greater than $r_1$. 
%We need to ensure that $\srsplus$ is commutative, therefore, if $r_1 = r_2$, the result is $(r_1, \min_{>}(S_1,S_2))$, where $>$ is some arbitrary but fixed total order on $2^{\lit(\mathbf{Q})}$. 
The product $(r_1, S_1) \srstimes (r_2, S_2)$ is defined as $(r_1 \cdot r_2, S_1 \cup S_2)$.
The weight function is given by $\alpha_O(l)= (p, \{l\})$ if $l=f$ for probabilistic fact $p\dd f$, $\alpha_O(l) = (1-p, \{l\})$ if $l= \lnot f$ for probabilistic fact $p\dd f$, and $\alpha_O(l)=(1, \{l\})$ otherwise. 
%The weight function $\alpha_O$ is such that $\alpha_O(l)= (p, \{l\})$ for each literal $l \in\lit(\mathbf{X}_O)$ if $l$ is a true probabilistic atom, $\alpha_O(l) = (1-p, \{l\})$ if $l$ is a false probabilistic atom and $\alpha_O(l)=(1, \{l\})$ if $l$ is not a probabilistic atom. 
The transformation function is the function $t(p)=(p, \emptyset)$. 

\paragraph{\textbf{Maximizing Expected Utility}}
Another second level probabilistic task is \emph{maximum expected utility}~\citep{dtproblog, derkinderen2020algebraic}, which introduces an additional set of variables $\mathbf{D}$ whose truth value can be arbitrarily chosen by a strategy $\sigma(\mathbf{D})=\mathbf{d}, \mathbf{d}\in int(\mathbf{D})$, and is neither governed by probability nor logical rules. A \emph{utility} function $u$ maps each literal $l$ to a reward $u(l)\in \mathbb{R}$ for  $l$ being true. %By default this reward is zero. (AK: functions don't need defaults)
%:  if $l\in lit(\mathbf{D})$, otherwise %, i.e. $l\in \mathbf{H}\backslash\lit(\mathbf{D})$,
%$u(l)=0$. 
\begin{define}[Maximum Expected Utility (MEU) Task]
\textbf{Given}
A program $\mathcal{L}=\mathbf{F}\cup R\cup \mathbf{D} $ with a utility function $u$\\
\textbf{Find} 
a strategy $\sigma^*$  that maximizes the expected utility:
\[\sigma^*= \arg \max_{\sigma} \sum_{M\in\mods(\mathbf{F}\cup R\cup\sigma(\mathbf{D}))} (\prod_{l\in M} \alpha(l)) (\sum_{l\in M} u(l)),\]
where the label $\alpha(l)$ of literal $l$ is as defined for $\mathit{SUCC}$.
\end{define}

\begin{example}[cont.]
Consider the program $\mathcal{L}_{EU}$ obtained from $\mathcal{L}_{ex}$ by replacing $0.4\dd a$ by a decision variable $?\dd a$,  with $u(c)=40$,  $u(\neg d)=20$ and $u(l)=0$ for all other literals. 
Setting $a$ to true, we have models $\{a,b,c,d\}$ with probability $0.6$ and utility $40$ and $\{a,c\}$ with  probability $0.4$ and utility $60$, and thus expected utility  $0.6\cdot 40 + 0.4\cdot 60 = 48$. Similarly, we have expected utility $0.6\cdot 0+0.4\cdot 20=8$ for setting $a$ to false. Thus, the MEU strategy sets $a$ true. 
\end{example}

\subparagraph{\textit{MEU as 2AMC.}} 
Solving MEU involves (1) summing expected utilities of models over the non-decision variables (with fixed truth values for $\mathbf{D}$), and (2) determining truth values of the atoms in $\mathbf{D}$ that maximize this inner sum. Thus, we have $\mathbf{X}_O=\mathbf{D}$ and $\mathbf{X}_I=\mathbf{H}\setminus\mathbf{D}$. 
%We can express the MEU task as a 2AMC problem as follows. The two problems are: (1) computing the expected utility of each strategy for $\mathbf{D}$ and (2) extracting a strategy that has the highest expected utility.

The inner problem is solved by the expected utility semiring $\mathcal{S}_I=EU$, with $\alpha_I$ mapping literal $l$ to $(p_l, p_l\cdot u(l))$ if $l$ is a probabilistic literal with probability $p_l$, and to $(1, u(l))$  otherwise.
%If the utility of an atom is not explicitly specified, it is zero. Here, $\mathbf{X}_I$ includes all the atoms that are not decision variables.

%$\mathbf{X}_O$ is thus the set of decision variables $\mathbf{D}$, 
The basis for solving the outer problem is the max-plus semiring $\mathcal{S}_{\max, +}$,  %$(\mathbb{R}\cup\{-\infty\}, max, \cdot, -\infty,1)$,
with $\alpha_O$ the utility function $u$, and the transformation function $\transform((p, pu)) = pu$, if $p\neq 0$ and $\transform((0, pu)) = -\infty$. This is extended to argmax using the same idea as for MAP.
%To obtain a maximizing assignment rather than the maximal value, we extend the semiring and functions with an additional component that tracks truth values using the same principle as for MAP inference.  
%such that $\alpha_O(l)= 0$ for each decision literal $l \in\lit(\mathbf{D})$ if $l$ does not have a utility and $\alpha_O(l) = u$ if $l$ has a utility. 

\paragraph{\textbf{Probabilistic Inference with Stable Models}}
A more recent second level probabilistic task is probabilistic inference with stable model semantics~\citep{totis2021smproblog, skryagin2021slash}. Inference of success probabilities reduces to a variant of weighted model counting, where the weight of a (stable) model of a program $\mathcal{L}=R\cup\mathbf{F}$ is normalized with the number of models sharing the same assignment $\mathbf{f}$ to the probabilistic facts $\mathbf{F}$:

% \[SUCC^{sm}(q)
% = \sum_{M\in \mathit{\mods}(\mathcal{L}\cup \omega), \omega\in\inter(\mathbf{F})}\frac{|\mods(\mathcal{L}\cup \omega \cup \{ q \})|}{|\mods(\mathcal{L}\cup \omega)|}\cdot\prod_{l\in M}w(l)\]
% \pietro{
\[ SUCC^{sm}(q)
= \sum_{\mathbf{h} \in \mods(\mathcal{L} \mid_{\mathbf{f}}), \mathbf{f}\in\inter(\mathbf{F})} \frac{|\mods(\mathcal{L} \mid_{\mathbf{f}\cup\{q\}})|}{|\mods(\mathcal{L} \mid_{\mathbf{f}})|}\cdot\prod_{l\in \mathbf{h}}\alpha(l),\]
where the label $\alpha(l)$ of literal $l$ is as defined for $\mathit{SUCC}$.

\begin{example}[cont.]
For $\mathcal{L}_{ex}$, each assignment $\mathbf{f}$ introduces exactly one stable model, i.e. SUCC$^{sm}$ equals SUCC.  
%Clearly, the program $\mathcal{L}_{ex}$ can also be interpreted as an \smp program. Furthermore, since each assignment $\mathbf{f}$ introduces exactly one stable model, the semantics also align in this case.
In the extended program $\mathcal{L}_{sm} = \mathcal{L}_{ex}\cup \{e \leftarrow \lpnot f. f \leftarrow \lpnot e. \}$, however, all assignments have two stable models, one where $e$ is added and one where $f$ is added. Therefore, for each assignment $\mathbf{F}=\mathbf{f}$ it holds that $|\mods(\mathcal{L}_{sm} \mid_{\mathbf{f}})|=2$ but $|\mods(\mathcal{L}_{sm} \mid_{\mathbf{f}\cup\{e\}})|=1$. Thus SUCC\textsuperscript{sm}$(e) = 0.5$.
\end{example}

%In \smp the inference task is reduced to a variant of the weighted model counting problem, $\widehat{\mathit{WMC}}$, where each satisfying assignment $M$ is associated to a normalization constant $\hat{w}(M_\omega)=\frac{1}{|\mods(\mathcal{L}_\omega)|}$:
%\[\widehat{\mathit{WMC}}_{\mathcal{L}}(q)= \sum_{M\in \mathit{\mods}(\mathcal{L}_\omega),M\models q,\omega\in\Omega(\mathcal{L})}\hat{w}(M)\cdot\prod_{l\in M}w(l)\]

% We again have a second level problem, since we need to perform weighted model counting twice in a nested manner. First each normalization constant is obtained solving a first level problem, namely counting the number of satisfying assignments that agree on the interpretation of the total choice described by the model $M$ and then we can perform the actual probabilistic inference using the normalization constants, which is another first level problem.

\subparagraph{\textit{SUCC\textsuperscript{sm} as 2AMC.}}
Computing SUCC\textsuperscript{sm} requires  (1) counting the number of models for a given total choice and those that satisfy the query $q$, and (2) summing the normalized probabilities. Thus, we have $\mathbf{X}_O=\mathbf{F}$ and $\mathbf{X}_I=\mathbf{H}\setminus\mathbf{F}$. 
$\mathcal{S}_I$ is the semiring over pairs of natural numbers, $\mathcal{S}_I=(\mathbb{N}^2, +, \cdot, (0,0), (1,1))$, where operations are component-wise. 
%Intuitively, given a total choice, a pair $(n_1,n_2)$ counts the models $n_1$, where the query is true, and $n_2$, the overall number of models. 
$\alpha_I$ maps $\lnot q$ to $(0,1)$ and all other literals to $(1,1)$. The first component thus only counts the models where $q$ is true ($|\mods(\mathcal{L} \mid_{\mathbf{f}\cup\{q\}})|$), whereas the second component counts all models ($|\mods(\mathcal{L} \mid_{\mathbf{f}})|$).
%is such that $\alpha_I(X)= (0,1)$ if $X\models \lnot q$, and $\alpha_I(X)= (1,1)$ otherwise. 
%This defines a model counting task over $\mathbf{X}_I$ yielding 
%($|\{M\,|\,M\in\mods(\mathcal{L}_{\mathbf{x}_I}), M\models q\}|$,$|\mods(\mathcal{L}_{\mathbf{x}_I})|)$ 
% \rafael{Should these be $\mathbf{x}_I$ instead of $\mathbf{X}_I$?}\pietro{yes, not convinced by the new notation in thius case though}. 
The transformation function is given by $t((n_1,n_2))=\frac{n_1}{n_2}$. 
%The normalization constants $\hat{w}(M)$ are then obtained by applying to the solution of problem 1 the transformation $t((n_1,n_2))=\frac{n_1}{n_2}$. 
%Problem 2 is solved by the probability semiring $\mathcal{R}_O=([0,1], +, \cdot, 0,1)$, and the function $\alpha_O$, which assigns %weight 1 (resp. 0) to all true (false) evidence variables, 
%weight 1 to each remaining non-probabilistic variable, and $\alpha_O(X)=p$, $\alpha_O(\lnot X)=1-p$, for each variable $X$ representing a probabilistic fact $p\dd f$
The outer problem then corresponds to usual probabilistic inference, i.e., $\mathcal{S}_O=\mathcal{P}$, and $\alpha_O$ assigns   $p$ and $1-p$ to the positive and negative literals for probabilistic facts $p\dd f$, respectively, and $1$ to all other literals.

%\subsubsection{\scp} 

% \scp~\citep{latour2017combining} is a system based on ProbLog which aims at modeling and reasoning over \emph{Stochastic Constraint Optimization Problems} (SCOPs). Similarly to \dtp, variables are divided into \emph{decision} variables, associated with a reward, and (mutually independent) \emph{stochastic} variables, associated with a probability value. The aim is to find an assignment to the decision variables, such that stochastic constraints and optimization criteria are satisfied. Constraints are of the kind $\sum_i r_i v_i \leq \theta$ and $\sum_i r_i v_i \geq \theta$, where $v_i$ is either a decision variable or a conditional probability $P_i(\varphi_i|D=d)$ that a stochastic Boolean formula $\varphi_i$ is true, given an assignment $d$ to decision variables $D$. $r_i\in\mathbb{R}$ is the corresponding reward and $\theta$ is a constant threshold. Constraints  express a bound on \emph{expected utilities}: the rewards for events, each of which could happen with a certain probability, are summed with respect to the corresponding decision variables. Therefore, the same mapping to a 2AMC as \dtp can be used for SCOPs. \cite{latour2017combining} in fact propose a knowledge compilation procedure for SDDs that satisfy specific vtree constraints (\emph{SMP vtrees}). The constraints guarantee that for each decision $D=d$ the objective value can be efficiently computed from the possibilities modelled by the stochastic variables. 

\section{Weakening $\mathbf{X}$-Firstness}
While any \twoAMC problem can be solved in polynomial time on an $\mathbf{X}_O$-first sd-DNNF representing the logical theory $\mathcal{T}$, such an sd-DNNF can be much bigger than the smallest (ordinary) sd-DNNF for $\mathcal{T}$, as the $\mathbf{X}$-firstness may severely restrict the order in which variables are decided. In the following, we show that for a wide class of transformation functions, we can exploit the logical structure of the theory to relax the $\mathbf{X}_O$-first property. 

Recall that a \twoAMC task includes an AMC task for every assignment $\mathbf{x}_O$ to the outer variables, which sums over all assignments $\mathbf{x}_I$ to the inner variables that extend $\mathbf{x}_O$ to a model of the theory. Consider the CNF $\mathcal{T} = \{y \vee \neg x, \neg y \vee x\}$ and let $\mathbf{X}_O = \{y\}$ and $\mathbf{X}_I = \{x\}$. The value of the outer variable $y$ already determines the value of the inner variable $x$. Distributivity allows us to pull $x$ out of every inner sum, as each such sum only involves one of the literals for $x$. If it does not matter whether we first apply the transformation function and then multiply or the other way around, we  have a choice between keeping $x$ in the inner semiring, or pushing its transformed version to the outer semiring. 
Thus, we can decide between  
%Then, we can decide, whether we want 
an $\mathbf{X}_O$-first or an $\mathbf{X}_O\cup\{x\}$-first sd-DNNF. Naturally, the more such variables we have, the more freedom we gain. 

This situation might also occur after we have already decided some of the variables. Consider the CNF $\mathcal{T}' = \{z \vee y \vee \neg x, z \vee \neg y \vee x\}$ and let $\mathbf{X}_O = \{z,y\}$ and $\mathbf{X}_I = \{x\}$. If we set $z$ to true, both $y$ and $x$ can take any value, therefore the value of $x$ is not determined by $z$ and $y$. However, if we set $z$ to false, 
we are in the same situation as above and can move $x$ to $\mathbf{X}_O$ on a local level. %Thus, we can apply the same trick of moving $x$ to $\mathbf{X}_O$ on a local level.

We formalize this, starting with definability to capture when a variable is determined by others.
\begin{define}[Definability~\citep{lagniez2016improving}]
A variable $a$ is \emph{defined} by a set of variables $\mathbf{X}$ with respect to a theory $\mathcal{T}$ if for every assignment $\mathbf{x}$ of $\mathbf{X}$ it holds that $\mathbf{x} \cup \mathcal{T} \models a$ or $\mathbf{x} \cup \mathcal{T} \models \neg a$. 
We denote the set of variables that are not in $\mathbf{X}$ and defined by $\mathbf{X}$ with respect to $\mathcal{T}$ by $\mathbf{D}(\mathcal{T}, \mathbf{X})$.
\end{define}
%For our purposes, it is important to note that when $a$ is defined by $\mathbf{X}$, then for any assignment $\mathbf{x}$ on $\mathbf{X}$ it holds that every satisfying assignment $\mathbf{a}$ to the variables $\mathbf{A}$ of $\mathcal{T}$ such that $\mathbf{a}$ and $\mathbf{x}$ agree on $\mathbf{X}$ has the same value for $a$.
\begin{example}[cont.]
In $\mathcal{L}_{ex}$ the atoms $c$ and $d$ are defined by $\{a,b\}$ since $c$ holds iff $a$ holds, and $d$ holds iff $b$ holds.
\end{example}

\begin{define}[$\mathbf{X}$-Firstness Modulo Definability]
Given an NNF $n$ on variables partitioned into $\mathbf{X},\mathbf{Y}$, we say an internal node $n_i$ of $n$ is pure \emph{modulo definability} if $\mathit{Vars}(n_i) \subseteq \mathbf{X} \cup \mathbf{D}(n_i, \mathbf{X})$ or $\mathit{Vars}(n_i) \subseteq \mathbf{Y}$ and mixed \emph{modulo definability}, otherwise. 
$n$ is an $\mathbf{X}$-first NNF \emph{modulo definability}, $\mathbf{X}/\mathbf{D}$-first NNF for short, if for each of its and-nodes $n_i$ either all children of $n_i$ are pure modulo definability, or one child of $n_i$ is mixed modulo definability and all other children $n_j$ of $n_i$ are pure modulo definability and $\mathit{Vars}(n_j) \subseteq \mathbf{X}\cup \mathbf{D}(n_i, \mathbf{X})$.
\end{define}
The intuition here is that we can decide variables from $\mathbf{Y}$ earlier, if they are defined by the variables in $\mathbf{X}$ in terms of the theory $\mathcal{T}$ conditioned on the decisions we have already made. Thus, we only decide the variables in $\mathbf{X}$ first \emph{modulo definability}.

% \begin{define}[$\mathbf{X/D}$-Constrainedness]
% Given two sets  $\mathbf{X}$ and $\mathbf{D}$ %\ak{should they be disjoint? or is overlap ok?} \rafael{anything is fine here.} 
% of variables, a vtree node $v'$ of a vtree $v$ is \emph{$\mathbf{X/D}$-constrained} ($\mathbf{X}$-constrained modulo $\mathbf{D}$) if 
% \begin{itemize}
%     \item $v'$ occurs on the rightmost path of $v$ and
%     \item for the set $\mathbf{S}$ of variables that occur in $v$ but not in the subtree rooted in $v'$, it holds that $\mathbf{X} \subseteq \mathbf{S} \subseteq \mathbf{X} \cup \mathbf{D}$.
% \end{itemize}
% A vtree that has an $\mathbf{X}/\mathbf{D}$-constrained node is called \emph{$\mathbf{X}/\mathbf{D}$-constrained} and an SDD that is normalized for an \emph{$\mathbf{X}/\mathbf{D}$-constrained} vtree is \emph{$\mathbf{X}/\mathbf{D}$-constrained}.
% \end{define}
\begin{example}[cont.]
In Figure~\ref{fig:sddnnfs}, the left NNF is an $\{a,b\}$-first NNF and therefore also an $\{a,b\}/\mathbf{D}$-first NNF. The right NNF is not an $\{a,b\}$-first NNF but an $\{a,b\}/\mathbf{D}$-first NNF since $\mathbf{D}(\mathcal{L}_{ex}, \{a,b\})$ contains $c$.
% In Figure~\ref{fig:vtrees}, the left vtree is $\mathbf{X}$-constrained and therefore also $\mathbf{X/D}$-constrained for any $\mathbf{D}$. The right vtree is not $\mathbf{X}$-constrained, but $\mathbf{X/D}$-constrained for any $\mathbf{D}$ that contains $c$.
%Consider again the vtrees in Figure~\ref{fig:vtrees}. Since the left vtree is $\mathbf{X}$-constrained it is also $\mathbf{X/D}$-constrained for any $\mathbf{D}$. Also the right vtree is $\mathbf{X/D}$-constrained for any $\mathbf{D}$ that contains $c$, even though it is not $\mathbf{X}$-constrained.
\end{example}
The following lemma generalizes this example from $2$ to $n$ pairs of equivalent variables.
%This example can be generalized to obtain the following Lemma showing the exponential size separation.
\begin{lemma}[Exponential Separation]
\label{lem:size_sep}
Let $\mathcal{T} = \bigwedge_{i = 1}^{n} X_i \leftrightarrow Y_i$, $\mathbf{X} = \{X_1, \dots, X_n\}$ and $\mathbf{D} = \{Y_1, \dots, Y_n\}$, then the size of the smallest $\mathbf{X}$-first sd-DNNF for $\mathcal{T}$ is exponential in $n$ and the size of the smallest $\mathbf{X}/\mathbf{D}$-first sd-DNNF for $\mathcal{T}$ is linear in $n$.
\end{lemma}
\begin{proof}[Proof (Sketch)]
Since $\mathbf{D}(\mathcal{T}, \mathbf{X}) = \mathbf{Y}$, every sd-DNNF for $\mathcal{T}$ is an $\mathbf{X}/\mathbf{D}$-first sd-DNNF. As  $\mathcal{T}$ has treewidth 2,  there exists an sd-DNNF of linear size. On the other hand, an $\mathbf{X}$-first sd-DNNF must contain a node that is equivalent to $\mathcal{T}\mid_{\mathbf{x}}$ for each of the $2^{|\mathbf{X}|}$ assignments $\mathbf{x} \in \inter(\mathbf{X})$.
\end{proof}
We see that $\mathbf{X}/\mathbf{D}$-first sd-DNNFs can be much smaller than $\mathbf{X}$-first sd-DNNFs, even on very simple propositional theories.
%Obviously, we cannot take any set $\mathbf{D}$ and still expect \twoAMC to be solvable on SMP($\mathbf{X}$)/$\mathbf{D}$ sd-DNNFs. However, if we take $\mathbf{D}$ to be the set of variables \emph{defined} by $\mathbf{X}$, as in Lemma~\ref{lem:size_sep}, 
It remains to show that we maintain tractability. As in the beginning of this section, we want to regard defined inner variables as outer variables. For this to work it must not matter whether we first multiply and then apply the transform $\transform$ or the other way around, i.e., $\transform$ must be a \emph{homomorphism} for the multiplications of the semirings. 
\begin{define}[Monoid Homomorphism, Generated Monoid]
Let $\mathcal{M}_i = (M_i, \odot^i, e_{\odot^i})$ for $i = 1,2$ be monoids. Then a \emph{monoid homomorphism} from $\mathcal{M}_1$ to $\mathcal{M}_2$ is a function $f: M_1 \rightarrow M_2$ such that 
\begin{align*}
 \text{for all $m, m' \in M_1$ } & & f(m \odot^1 m') = f(m) \odot^2 f(m') & & \text{ and } & & f(e_{\odot^1}) = e_{\odot^2}.    
\end{align*}
Furthermore, for a subset $M' \subseteq M$ of a monoid $\mathcal{M} = (M, \odot, e_{\odot})$ the \emph{monoid generated by $M'$}, denoted $\langle M' \rangle_{\mathcal{M}}$, is $\mathcal{M}^* = (M^*, \odot, e_{\odot})$, where $M' \subseteq M^*$ and $M^* $ is the subset minimal set such that $\mathcal{M}^*$ is a monoid.
\end{define}

\begin{example}[cont.]
Consider again the \twoAMC instance $A_{ex}$ from Example~\ref{ex:twoamc}. Since $a$ is defined in terms of $c$, we want to argue that the following equality holds, allowing us to see $a$ as an outer variable:
\begin{align*}
    &\textstyle\max_{\mathbf{x_O} \in \inter(\{c\})} \alpha_{O}(\mathbf{x_O})\cdot id\left(\sum_{\mathbf{x_I} \in \mods(\mathcal{L}_{ex}\mid_{\mathbf{x}_O})} \prod_{y \in \mathbf{x_I}} \alpha_I(y)\right)\\
   = & \textstyle\max_{\mathbf{x_O} \in \inter(\{c\}),\mathbf{x_{O'}} \in \inter(\{a\})}\alpha_{O}(\mathbf{x_O})\cdot id(\alpha_{I}(\mathbf{x_{O'}}))\cdot id\left(\sum_{\mathbf{x_I} \in \mods(\mathcal{L}_{ex}\mid_{\mathbf{x}_O \cup \mathbf{x}_{O'}})} \prod_{y \in \mathbf{x_I}} \alpha_I(y)\right)
\end{align*}
Here, this is easy to see since $id$ is a homomorphism between the monoid $([0,1], \cdot)$ of the inner probability semiring and the monoid  $(\mathbb{R}_{\geq 0}, \cdot)$ of the outer max-times semiring, as $id: [0,1]\rightarrow \mathbb{R}_{\geq 0}$, for any $p,q\in[0,1]$, $id(p \cdot q) = p\cdot q = id(p) \cdot id(q)$, and $id(1)=1$. 
%However, there are other transformation functions where $\transform$ is not a homomorphism on all values but only on the observable ones.
%Here, the observable values are the values of the defined literals, i.e., $\alpha_I(a) = 0.4, \alpha_I(\lnot a) = 0.6$ and the possible values for the inner sum if we take out and assign a subset of the defined variables, e.g., $\sum_{\mathbf{x_I} \in \mods(\mathcal{L}_{ex}\mid_{\{c\}\cup\{a\}})} \prod_{y \in \mathbf{x_I}} \alpha_I(y)$. These correspond to the values in the first and second line of Definition~\ref{def:obs}, respectively.
\end{example}
In general, instead of applying the transform to a sum of products of literal labels for a set of variables, we want to apply it independently to (1) the literal labels of a defined variable and (2) the inner sum restricted to the remaining variables.  It is therefore sufficient if the equality is valid for the monoid generated by the values we encounter in these situations, rather than for all values from the inner monoid's domain. As we will illustrate for MEU at the end of this section, the transform of some 2AMC tasks only satisfies this weaker but sufficient condition.
%Moreover, it is sufficient if $\transform$ is a homomorphism on the values that it can be applied to, when we move defined inner variables to the outer variables. These values are literal weights of defined variables or weights of inner sums, where a subset of the defined variables is missing. We capture them as follows:
The following definition captures this idea, where the two subsets of $O(A)$ correspond to the values observed in  cases (1) and (2), respectively:
\begin{define}[Observable Values]
\label{def:obs}
Let $A = (\mathcal{T}, \mathbf{X}_I, \mathbf{X}_O, \alpha_I, \alpha_O, \mathcal{S}_I, \mathcal{S}_O, \transform)$ be a \twoAMC instance. Then the set of \emph{observable values} of $A$, denoted $O(A)$ is
\begin{align*}
    &\{ \alpha_I(l) \mid \mathbf{x}_O \in \inter(\mathbf{X}_O), y \in \mathbf{D}(\mathcal{T}\cup \mathbf{x}_O, \mathbf{X}_O), l \in \{y, \neg y\}\}\\
    \cup &\{\srbplus^{I}_{\mathbf{x}_I \in \mods(\mathcal{T} \mid_{\mathbf{x}_O})} \srbtimes^{I}_{y \in \mathbf{x}_I} \alpha_I(y) \mid \mathbf{x}_O \in \inter(\mathbf{X}_O \cup \mathbf{D}^*), \mathbf{D}^* \subseteq \mathbf{D}(\mathcal{T}\cup \mathbf{x}_O, \mathbf{X}_O)\}.
\end{align*}
\end{define}

%Intuitively, this set contains all values from $R_I$ that might be given to $\transform$ during evaluation.

With this in mind, we can state our main result.
%\begin{theorem}[Tractable AMC with $\mathbf{X}/\mathbf{D}$-first sd-DNNFs]
%\label{thm:tract2}
%Let
%\begin{itemize}
%    \item $A = (\mathcal{T}, \mathbf{X}_I, \mathbf{X}_O, \alpha_I, \alpha_O, \mathcal{R}_I, \mathcal{R}_O, \transform)$ be a \twoAMC instance,
    %\item $\mathbf{D} \subseteq \mathbf{X}_I$ be the set of variables defined by $\mathbf{X}_O$ w.r.t.\ $\mathcal{T}$, 
%    \item $\mathcal{T}$ be an $\mathbf{X}/\mathbf{D}$-first sd-DNNF
 %   \item $\mathcal{M} = \langle O(A) \rangle_{(R_I, \srstimes^{I}, e_{\srstimes^{I}})}$ be the monoid generated by the observed values and
  %  \item $\transform$ be a monoid homomorphism from $\mathcal{M}$ to $(R_O, \srstimes^{O}, e_{\srstimes^{O}})$.
%\end{itemize} 
%Then we can compute the value of $A$ in polynomial time in the size of $\mathcal{T}$.
%\end{theorem}
\begin{theorem}[Tractable AMC with $\mathbf{X}/\mathbf{D}$-first sd-DNNFs]
\label{thm:tract2}
The value of a \twoAMC instance $A = (\mathcal{T}, \mathbf{X}_I, \mathbf{X}_O, \alpha_I, \alpha_O, \mathcal{S}_I, \mathcal{S}_O, \transform)$ can be computed in polynomial time, assuming constant time semiring operations, under the following conditions:
\begin{itemize}
    \item $\mathcal{T}$ is an $\mathbf{X}/\mathbf{D}$-first sd-DNNF
    \item $\transform$ is a homomorphism from the monoid $\langle O(A) \rangle_{(R_I, \srstimes^{I}, e_{\srstimes^{I}})}$ generated by the observable values to $(R_O, \srstimes^{O}, e_{\srstimes^{O}})$.
\end{itemize} 
\end{theorem}
\begin{proof}[Proof (Sketch)]
The proof of this theorem exploits (a) distributivity and the form of the transformation function to move defined variables to the outer semiring as outlined at the start of this section, and (b) the fact that when we decide an outer variable, then we get two new \twoAMC instances, where the theory $\mathcal{T}$ is conditioned on the truth of the decided variable. On these new instances, we can also use definability in the same fashion as before.
\end{proof}

Despite the fact that checking which variables are defined for each partial assignment $\mathbf{x}$ is not feasible, as checking definability is co-NP-complete and there are more than $2^{|\mathbf{X}_O|}$ partial assignments, this result does have implications for constrained KC in practice. 

%This brings us to the question of how we can use the weakening of $\mathbf{X}$-firstness to $\mathbf{X}/\mathbf{D}$-firstness in practice. Checking which variables are defined for each partial assignment $\mathbf{x}$ is not a good strategy: Checking definability is co-NP-complete and there are more than $2^{|\mathbf{X}_O|}$ partial assignments. Nevertheless, we can fruitfully use this result in two fashions to obtain smaller sd-DNNFs, on which we can solve \twoAMC instances. 

Firstly, we can check which variables are defined by $\mathbf{X}_O$ in terms of the whole theory $\mathcal{T}$,  and use this to generate a variable order for compilation that leads to an $\mathbf{X}/\mathbf{D}$-first sd-DNNF with a priori guarantees on its size. We discuss this in Section~\ref{sec:implementation}.

Secondly, as entailment is a special case of definability, 
Theorem~\ref{thm:tract2} justifies the use of unit propagation during compilation, which dynamically adapts the variable order when variables are entailed by the already decided ones, and thus may violate $\mathbf{X}$-firstness.
%recall entailment is a special case of definability. Therefore, this Theorem allows us to immediately decide variables, that are found to be entailed by unit propagation during compilation. This is not possible when we compile to $\mathbf{X}$-first sd-DNNF, which include $\mathbf{X}$-constrained SDDs, the current standard in the literature~\citep{latour2017combining,derkinderen2020algebraic,bellodi2020map}.

%Before we can apply Theorem~\ref{thm:tract2} to solve \twoAMC tasks more efficiently we still need to make sure that the preconditions are satisfied. Thus, we need the transformation function to be a monoid homomorphism on the observable values. 
We still need to verify that our three example tasks satisfy the preconditions of Theorem~\ref{thm:tract2}, i.e., their transformation functions are monoid homomorphisms on the observable values. 
For MAP and SUCC\textsuperscript{sm} this is easy to prove, as the transformation is already a homomorphism on all values. 
For MEU, however, the restriction to the monoid generated by the observable values is crucial, as the  transformation function is only a homomorphism for tuples $(p, pu)$ with $p \in \{0,1\}$, which may not be the case in general. In the \dtp setting, however, assignments to decision facts and probabilistic facts are independent, and every assignment to both sets extends to a single model of the theory. Together with the $(1,u)$-labels of the remaining atoms, this ensures that the result of the inner sum is of the form~$(1,x)$, and MEU thus meets the criteria. 

\section{Implementation}\label{sec:implementation}
The general pipeline of PLP solvers takes a program, grounds it and optionally simplifies it or breaks its cycles. Using standard KC tools, this program is then  compiled into a tractable circuit either directly or via conversion to CNF.
We extended this pipeline in \aspmc~\citep{eiter2021treewidth} and \smp~\citep{totis2021smproblog} to compile programs, via CNF, into $\mathbf{X}/\mathbf{D}$-first circuits. 

To obtain $\mathbf{X}/\mathbf{D}$-first circuits we 
%The compilation itself makes use of standard knowledge compilation tools, namely c2d~\cite{Darwiche04} and miniC2D~\cite{oztok2015top}. To ensure that the compiled circuit satisfies $\mathbf{X}/\mathbf{D}$-firstness we 
need to specify a variable order in which all variables in $\mathbf{X}$ are decided first modulo definedness. 
Preferably, the chosen order should also result in efficient compilation and thus a small circuit.
\cite{korhonen2021integrating} have shown how to generate  variable orders  from \emph{tree decompositions} of the primal graph of the CNF that result in (unconstrained) sd-DNNFs whose size is bounded by the \emph{width} of the decomposition. We adapt this result to our constrained setting, where we need to ensure that the tree decomposition satisfies an additional property that allows compilation to essentially consider the non-defined inner variables and the outer variables independently. We first define the necessary concepts. 

%Additionally, we want \emph{good} variable orders, i.e., those that lead to a small circuit, which can be compiled fast. \cite{korhonen2021integrating} have shown that for sd-DNNFs we can generate good variable orders from \emph{tree decompositions} of the primal graph of the CNF, since these lead to circuits, whose size is bounded by the \emph{width} of the decomposition. Naturally, we cannot use every tree decomposition, when we want $\mathbf{X/D}$-first circuits. However, if we consider restricted TDs that intuitively ensure that we can handle the non-defined inner variables and the outer variables independently, then we can construct an $\mathbf{X}/\mathbf{D}$-first sd-DNNF with performance guarantees. Before we formalize the latter, we need the following definitions:
\begin{define}[Primal Graph, Tree Decomposition]%, Treewidth]
Let $\mathcal{T}$ be a CNF over variables $\mathbf{X}$. The \emph{primal graph} of $\mathcal{T}$, denoted by $\PRIM(\mathcal{T})$, is defined as $V(\PRIM(\mathcal{T})) = \mathbf{X}$ and $(v_1,v_2) \in E(\PRIM(\mathcal{T}))$ if $v_1, v_2$ occur together in some clause of~$\mathcal{T}$.

A \emph{tree decomposition (TD)} for a graph $G$ is a pair $(T, \chi)$, where $T$ is a tree and $\chi$ is a labeling of $V(T)$ by subsets of $V(G)$ s.t. (i) for all nodes $v \in V(G)$ there is $t \in V(T)$ s.t. $v \in \chi(t)$; (ii) for every $(v_1, v_2) \in V(E)$ there exists $t \in V(T)$ s.t. $v_1, v_2 \in \chi(t)$ and (iii) for all nodes $v \in V(G)$ the set of nodes $\{t \in V(T) \mid v \in \chi(t)\}$ forms a connected subtree of $T$.
The width of $(T, \chi)$ is $\max_{t \in V'} |\chi(t)| - 1$.
%The \emph{treewidth} of a graph is the minimal width of any of its TDs. 
\end{define}
%When we have a TD we can use it to determine the order in which variables should be decided during KC, giving us the following guarantee.
%With this in mind, 
Next, we show that restricted TDs allow $\mathbf{X}/\mathbf{D}$-first compilation with performance guarantees:
\begin{lemma}
\label{lem:XD-tree}
Let $\mathcal{T}$ be a CNF over variables $\mathbf{Y}$ and $(T,\chi)$ a TD of $\PRIM(\mathcal{T})$ of width $k$. 
%Then we can compute an SDD $\mathcal{T}$ that is equivalent to $\alpha$ in time $\mathcal{O}(2^{k}\cdot \text{poly}(|\alpha|))$.
Furthermore, let $\mathbf{X} \subseteq \mathbf{Y}$ and $\mathbf{D} = \mathbf{D}(\mathcal{T}, \mathbf{X})$. %\ak{we've used A rather than Y for the variables of T before -- let's switch to Y throughout, and use A for the algebraic stuff} \rafael{done}
If there exists $t^* \in V(T)$ such that (1) $\chi(t^*) \subseteq \mathbf{X} \cup \mathbf{D}$ and (2) $\chi(t^*)$ is a \emph{separator} of $\mathbf{X}$ and $\mathbf{Y}\setminus (\mathbf{X}\cup\mathbf{D})$, i.e., every path from $\mathbf{X}$ to $\mathbf{Y}\setminus (\mathbf{X}\cup\mathbf{D})$ in $\PRIM(\mathcal{T})$ uses a vertex from $\chi(t^*)$, then we can compile $\mathcal{T}$ into an $\mathbf{X}/\mathbf{D}$-first sd-DNNF in time $\mathcal{O}(2^{k}\cdot \text{poly}(|\mathcal{T}|))$. 
\end{lemma}
%This result involves the generation of a vtree from the TD, which we then use to construct the SDD. 
\begin{proof}[Proof (Sketch).]
The performance guarantee is due to~\cite{korhonen2021integrating} and holds when we decide the variables in the order they occur in the TD starting from the root.
%that we can generate a variable order from the TD such that the given performance guarantees are met during compilation. We additionally need to show that this variable can be chosen in such a manner that the resulting sd-DNNF is an 
$\mathbf{X}/\mathbf{D}$-firstness can be guaranteed by taking $t^*$ as the root of the TD and, thus, first deciding all variables in $\chi(t^*)$. % = \{S_1, \dots, S_n\}$. 
%If we first decide $S_1, \dots, S_n$, 
From condition (2) it follows that afterwards the CNF has decomposed into separate components, which either only use variables from $\mathbf{X} \cup \mathbf{D}$ or use no variables from $\mathbf{X}$. Thus, their compilation only leads to pure NNFs.% and we can generate the order on the remaining variables from $(T,\chi)$ like \cite{korhonen2021integrating}.
\end{proof}
%Thus, to allow efficient $\mathbf{X}/\mathbf{D}$-first compilation, we want to find such a restricted TD of small width. 
%This allows us to construct $\mathbf{X}/\mathbf{D}$-first sd-DNNFs with performance guarantees. 

To find a TD of small width for which the lemma applies, we proceed as follows. 
%For this, we practically proceed as follows. 
Given a CNF $\mathcal{T}$ and partition $\mathbf{X}_I, \mathbf{X}_O$ of the variables in inner and outer variables, we first compute $\mathbf{D}(\mathcal{T}, \mathbf{X}_O)$ by using an \NP-oracle. \cite{lagniez2016improving} showed that this is possible. 
As the width of a suitable  TD is at least the size of the separator minus one,  we first approximate a minimum size separator $\mathbf{S} \subseteq \mathbf{X} \cup \mathbf{D}(\mathcal{T}, \mathbf{X})$ using clingo~\citep{gebser2014clingo} with a timeout of 30 seconds. 
To ensure that the TD contains a node $t'$ with $\mathbf{S}\subseteq \chi(t')$, we add the clique over the vertices in $\mathbf{S}$ to the primal graph, i.e., we generate a tree decomposition $(T, \chi)$ of $\PRIM(\mathcal{T}) \cup \text{Clique}(\mathbf{S})$. For this, we use flow-cutter~\citep{Dell17} with a timeout of 5 seconds. The resulting TD either already contains a node with $\mathbf{S}= \chi(t')$ and thus satisfies the preconditions of Lemma~\ref{lem:XD-tree}, or can be modified locally to one that does by splitting the node with $\mathbf{S}\subset \chi(t')$. 

\paragraph{aspmc:}\footnote{\textsf{aspmc} is open source and available at \href{https://github.com/raki123/aspmc}{github.com/raki123/aspmc}}
%We implemented this strategy to generate variable orders in \aspmc. The variable order 
We use the above strategy to generate a variable order, which 
is then given to c2d~\citep{Darwiche04} or miniC2D~\citep{OztokDarwiche15} together with the CNF to compile an  $\mathbf{X}/\mathbf{D}$-first circuit, on which we evaluate the \twoAMC task. We stress that c2d -- contrary to miniC2D -- always uses unit propagation during compilation, and thus %. Thus, c2d cannot compile $\mathbf{X}$-first circuits but we
can only be used  due to Theorem~\ref{thm:tract2}.
%After compilation, we evaluate the \twoAMC instance in polynomial time in the size of the resulting circuit.

Currently, \aspmc supports \dtp, MAP and \smp programs as inputs.

\paragraph{smProbLog:} \footnote{\smp is open source and available at \href{https://github.com/PietroTotis/smProblog}{github.com/PietroTotis/smProblog}}The \smp implementation of \cite{totis2021smproblog} uses \textsc{dsharp}~\citep{DBLP:conf/aaai/AzizCMS15} 
to immediately compile the logic program to a  d-DNNF, which prevents a direct application of our new techniques. Our adapted version obtains  a CNF $\mathcal{T}$ and  variable ordering as in \aspmc, which it then compiles to SDD~\citep{darwiche2011sdd} using the \emph{PySDD} library\footnote{\href{https://github.com/wannesm/PySDD}{github.com/wannesm/PySDD}} as in standard ProbLog. SDDs can be seen as a special case of sd-DNNFs. The main difference is that for each branch the variables are decided in the same order. 

%The CNF is compiled into an SDD with the desired variable order by means of the \emph{PySDD} library\footnote{https://github.com/wannesm/PySDD}.  
%The original version skips the CNF conversion and compiles the logic program directly into a d-DNNF circuit by means  of the \textsc{dsharp} compiler~\citep{DBLP:conf/aaai/AzizCMS15}. Therefore, enforcing the desired variable order in this kind of pipeline is not possible.
%\ak{this now sounds even more like a variant of aspmc and less like a modification of smproblog than before}

\section{Experimental Evaluation}
Our experimental evaluation addresses the following questions:
\begin{itemize}
    \item[Q1:] How does exploiting definedness influence the efficiency of \twoAMC solving? 
    \item[Q2:] How does \aspmc compare to task-specific solvers from the PLP literature?
    \item[Q3:] How does our second level approach compare to the first level approach when definedness reduces \twoAMC to AMC?
\end{itemize}
\subsection{General Setup}
To answer these questions, we consider logic programs from the literature on the three example \twoAMC tasks. To eliminate differences on how different solvers handle n-ary random choices (known as annotated disjunctions) and 0/1-probabilities, we normalize probabilistic programs to contain only probabilistic facts and normal clauses, and replace all probabilities by values chosen uniformly at random from $0.1, 0.2, ..., 0.8, 0.9$.

For MAP, we use the growing head, growing negated body, blood and graph examples of \cite{bellodi2020map}, with a subset of the probabilistic facts of uniformly random size as query atoms and the given evidence. 
For MEU, we use the Bayesian networks provided by \cite{derkinderen2020algebraic} as well as the viral marketing example from~\cite{dtproblog} on randomly generated power law graphs that are known to resemble social networks~\citep{barabasi2003scale}. 
For SUCC\textsuperscript{sm}, we use an example  from~\cite{totis2021smproblog} that introduces non-probabilistic choices into the well-known smokers example \citep{problog2system}. 

Besides this basic set, we also use the original smokers example, where SUCC\textsuperscript{sm} reduces to SUCC, for Q3. For Q1, we use the grid graphs of \cite{problog2system} as an additional MAP benchmark. Here, we control definedness by choosing the MAP queries as the probabilistic facts for the edges that are reachable in $k$ steps from the top left corner, for all possible values of $k$, and use the existence of a path from the top left to the bottom right corner as evidence.

We compare the following systems:
\begin{description}
    \item[\aspmc] with %a 5 second timeout to compute TDs and 
    c2d as default knowledge compiler
    \item[ProbLog] in different versions: ProbLog2 (version 2.1.0.42) in MAP and default (SUCC) mode, the implementation of \cite{derkinderen2020algebraic} for MEU, and our implementation of \smp for SUCC\textsuperscript{sm}
    \item[PITA] (\cite{bellodi2020map}, version 4.5, included in SWI-Prolog) for MAP and MEU
    \item[clingo] (\cite{gebser2014clingo}, version 5.5.0.post3) as an indicator of how an enumeration based approach not using knowledge compilation might perform. Note that this approach does not actually compute the \twoAMC value of the formula, but only enumerates models.
\end{description}

We limit each individual run of  a system to 300 seconds and 4Gb of memory. When plotting running time per instance for a specific solver, we always sort instances by ascending time for that solver, using 300 seconds for any instance that did not finish successfully within these bounds.

The instances, results and benchmarking scripts are available at \href{https://github.com/raki123/CC}{github.com/raki123/CC}.

\subsection{Results}
\paragraph{\textbf{Q1}: How does exploiting definedness influence the efficiency of \twoAMC solving? }
To answer the first question, we use all proper \twoAMC benchmarks, and focus on the \twoAMC task itself, i.e., we start from the labeled CNF corresponding to the instance. We consider four different settings obtained by independently varying two dimensions: constraining compilation to either $\mathbf{X}$-first or $\mathbf{X}/\mathbf{D}$-first, and compiling to either sd-DNNF using c2d or to SDD using miniC2D. 

On the left of Figure~\ref{fig:widths}, we plot the width of the tree decompositions the solver uses to determine the variable order in the $\mathbf{X}$-first or $\mathbf{X}/\mathbf{D}$-first case, respectively. Recall from Lemma~\ref{lem:XD-tree} that this width appears in the exponent of the compilation time bound. The \emph{optimal} tree decomposition's width in the  $\mathbf{X}/\mathbf{D}$-first case is at most that of the $\mathbf{X}$-first case, and would thus result in points on or below the black diagonal only. In practice,
we observe many points close to the diagonal, with two notable exceptions. MAP instances with high width tend to be slightly above the diagonal, whereas MAP grids are mostly clearly below the diagonal. These results can be explained by the shape of the problems and the fact that we only approximate the optimal decomposition, as this is a hard task. We note that for many of the benchmarks, the amount of  variables defined in terms of the outer variables (decision variables for MEU, query variables for MAP, probabilistic facts for SUCC$^{SM}$) is limited. The exception are the MAP grids, where the choice of queries entails definedness. %, and we indeed see a positive effect also in practice.  

We plot the same data restricted to the instances solved within the time limit on the right of Figure~\ref{fig:widths}, along with summary statistics on the number of instances solved for three ranges of $\mathbf{X}/\mathbf{D}$-width in the caption. We observe that almost no instances with $\mathbf{X/D}$-width  above 40 are solved. At the same time, almost all instances with $\mathbf{X/D}$-width  below 20 are solved, including many cases with $\mathbf{X}$-width above 40, where we thus see a clear benefit from exploiting definedness.  

%We also considered whether the $\mathbf{X/D}$-width not only provides a better upper-bound on the running time via Lemma~\ref{lem:XD-tree} but also is a better indicator of whether an instance can be solved. The right plot in Figure~\ref{fig:widths} shows only the solved instances. We see that even instances with high $\mathbf{X}$-width can be solved if the $\mathbf{X/D}$-width is low. On the other hand, if the $\mathbf{X/D}$-width is high, i.e., above 40, our chances of solving it within a timeout of 300 seconds are slim. This is supported by Table~\ref{tab:solved_widths}, where we compared the number of solved instances with the total number of instances for different ranges of the $\mathbf{X/D}$-width. We see that indeed almost all the instances of $\mathbf{X/D}$-width up to 20 are solved, around half of the instances with width between 21 and 40 are solved and almost none of the instances with higher width are solved. 

\begin{figure}% Q1 widths
    \centering
    \includegraphics[width=\textwidth]{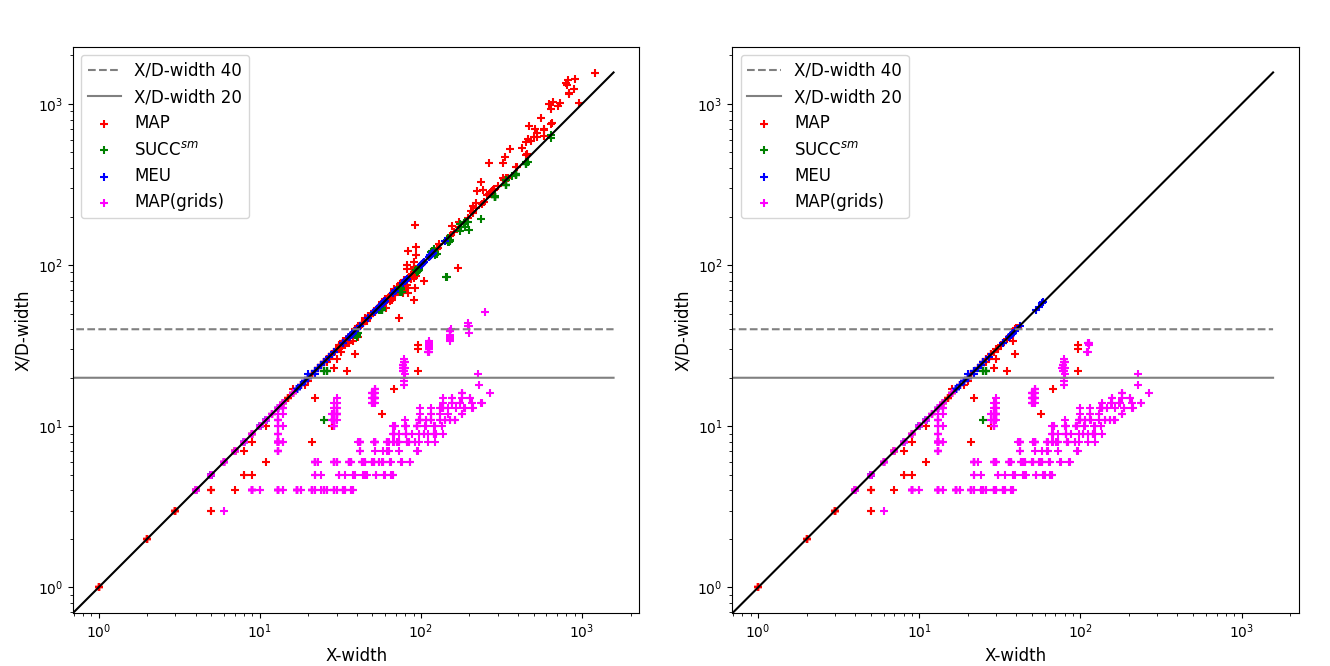}
    \caption{Q1: Comparison of tree decomposition width for $\mathbf{X}$-first and $\mathbf{X}/\mathbf{D}$-first variable order, across all \twoAMC instances (left) and for solved \twoAMC instances only (right). We solve $822$ of the $825$ instances with $\mathbf{X}/\mathbf{D}$-width at most $20$, $118$ of the $219$ instances with $\mathbf{X}/\mathbf{D}$-width between $21$ and $40$, and $7$ of the $353$ instances with $\mathbf{X}/\mathbf{D}$-width above $40$. 
    }
    \label{fig:widths}
\end{figure}
\begin{figure}% Q1 widths
    \centering
    \includegraphics[width=0.49\textwidth]{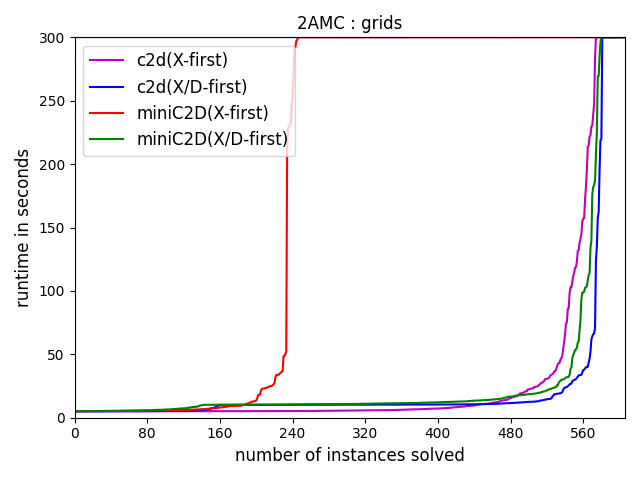}
    \includegraphics[width=0.49\textwidth]{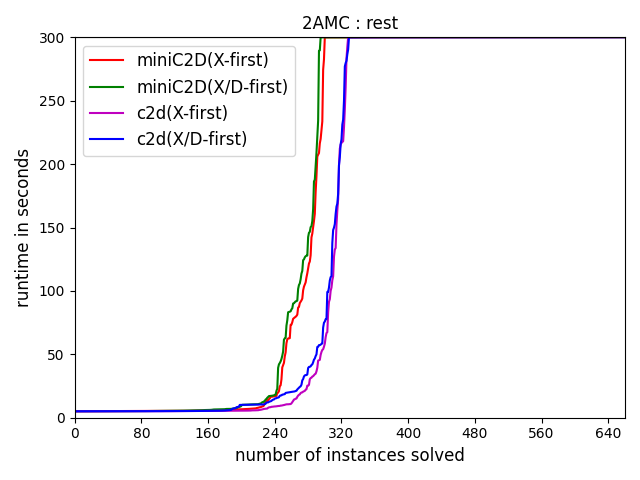}
    \caption{Q1: Running times per instance for different configurations on MAP grids (left) and all other \twoAMC instances (right).
    }
    \label{fig:eff_concom}
\end{figure}
%\begin{table}[tbh]
%    \centering
%    \begin{tabular}{l|c|c|c}
%        $\mathbf{X/D}$-width & 0 - 20 & 21 - 40 & $>$ 40 \\\hline
%        total number of instances & 825 & 219 & 353 \\
%        solved number of instances & 822 & 118 & 7 \\
%    \end{tabular}
%    \caption{Q1: The total number of instances and number of solved instances grouped by their $\mathbf{X/D}$-width upper bounds.}
%    \label{tab:solved_widths}
%\end{table}
In Figure~\ref{fig:eff_concom}, we plot the running times per instance for the different solvers. Given the width results, we distinguish between MAP grids and the remaining cases. 
On the grids, taking into account definedness results in clear performance gains when compiling SDDs (using miniC2D). On the other hand, compiling sd-DNNFs with c2d shows only a marginal difference between $\mathbf{X}$ and $\mathbf{X}/\mathbf{D}$ variable orders. The reason is that c2d implicitly exploits definedness even when given the $\mathbf{X}$-first order through its use of unit propagation. SDD compilation, on the other hand, cannot deviate from the given order, and only benefits from definedness if it is reflected in the variable order. On the other benchmarks, with fewer defined variables, the variable order has little effect within the same circuit class, but sd-DNNFs outperform SDDs, likely because unit propagation can also exploit context-dependent definedness. In the following we thus use c2d.%\ak{check terminology}

%\begin{figure}% Q1 running times
 %   \centering
 %   \includegraphics[width=0.45\textwidth]{efficiency_concom_grids.png}
 %   \includegraphics[width=0.45\textwidth]{efficiency_concom.png}
 %   \caption{Q1: Running times per instance for different configurations on MAP grids (left) and all other \twoAMC benchmarks (right).} 
%    Comparison of the runtimes of miniC2D and c2d on \twoAMC instances, when using a variable order generated with (X/D) or without (X) using definability.  }
  %  \label{fig:concom_eff}
%\end{figure}

\paragraph{\textbf{Q2}:  How does aspmc compare to task-specific solvers from the PLP literature?}
We first consider the efficiency of the whole pipeline from instance to solution on the MAP and MEU tasks, which are addressed by both ProbLog and PITA. %, albeit in both cases limited to probabilistic facts as MAP queries. 
From the plots of running times in Figure~\ref{fig:map_eff_per_set}, we observe that all solvers outperform clingo's model enumeration. ProbLog is  slower than both  \aspmc and PITA except on the MAP graphs. Among \aspmc and PITA, \aspmc outperforms PITA on MEU, and vice versa on MAP. This can be explained by the different overall approach to knowledge compilation taken in the various solvers. \aspmc always compiles a CNF encoding of the whole ground program, including all ground atoms, ProbLog compiles a similar encoding, but restricted to the part of the formula that is relevant to the task at hand, while PITA directly compiles a circuit for the truth of atoms of interest in terms of the relevant ground choice atoms. As MAP queries are limited to probabilistic facts, this allows PITA to compile a circuit for the truth of the evidence in terms of relevant probabilistic facts only, which especially for the graph setting can be significantly smaller than a complete encoding. For MEU, PITA needs one circuit per atom with a utility, which additionally need to be combined, putting PITA at a disadvantage.
%follows a different approach that directly compiles a formula using only labeled ground atoms that are relevant to the task at hand. As MAP queries are limited to probabilistic facts, this allows PITA to compile a formula describing the evidence in terms of probabilistic facts only, which especially for the graph setting can be significantly smaller than a complete encoding. For MEU, a more complex formula is needed as utilities are not restricted to facts, and PITA no longer benefits from the simplification. 

% If you make the final decision to keep the MEU and MAP in the plot like this, let me know because then I will make a .png with all three plots in it such that they are the same size.
\begin{figure}%Q2 MAP 
    \centering
    \includegraphics[width=0.49\textwidth]{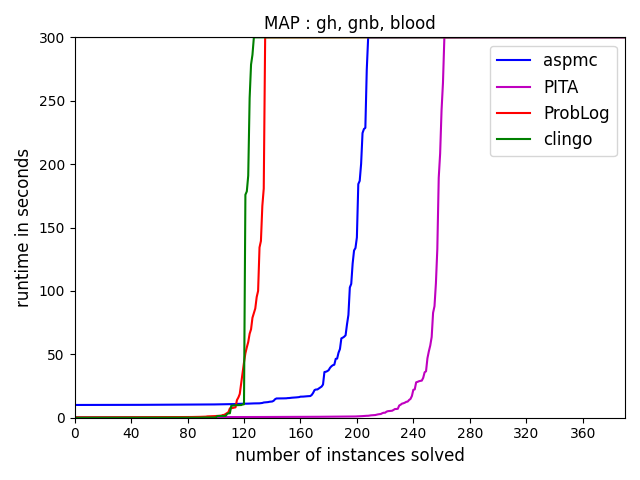}
    \includegraphics[width=0.49\textwidth]{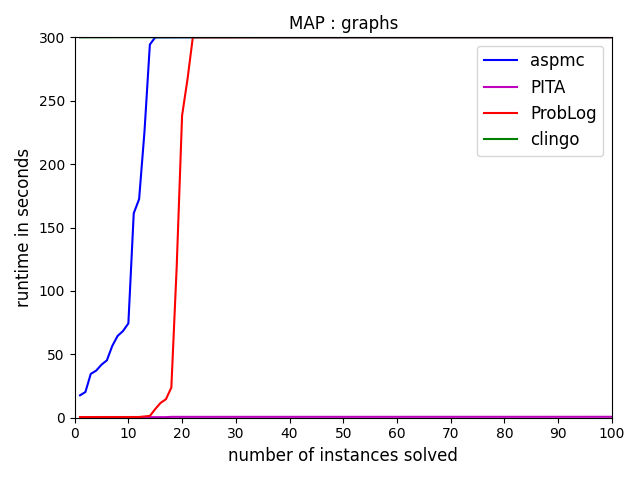}
    \includegraphics[width=0.49\textwidth]{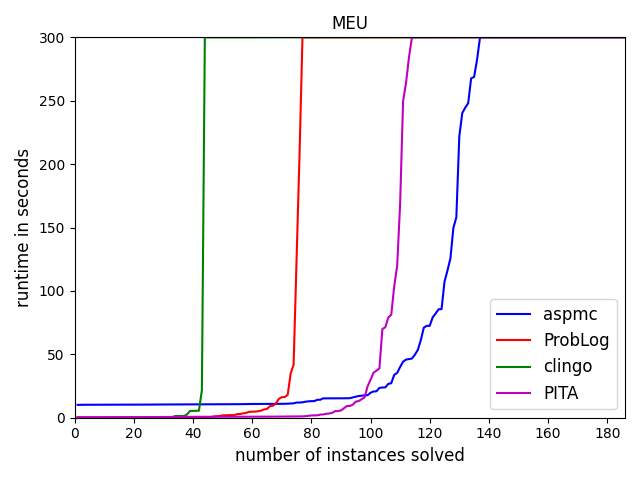}
    \includegraphics[width=0.49\textwidth]{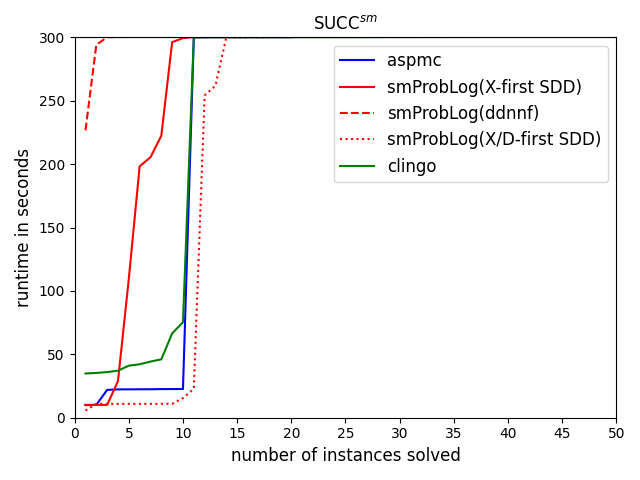}
    \caption{Q2: Running times of different solvers on MAP problem sets (top), indicated above each plot, MEU problems (bottom left) and SUCC\textsuperscript{SM} (bottom right).}
    \label{fig:map_eff_per_set}
\end{figure}

For SUCC\textsuperscript{sm}, we plot running times on the modified smokers setting for the clingo baseline, \aspmc and three variants of the dedicated ProbLog implementation, namely the original implementation of \cite{totis2021smproblog} that compiles to d-DNNF  as well as our modified implementation compiling to $\mathbf{X}$-first and $\mathbf{X}/\mathbf{D}$-first SDDs. These problems appear to be hard in general, but we observe a clear benefit from the constrained compilation enabled in our approach.

\begin{figure}% Q2 SUCCSM & Q3 SUcc(sm)
    \centering
    \includegraphics[width=0.49\textwidth]{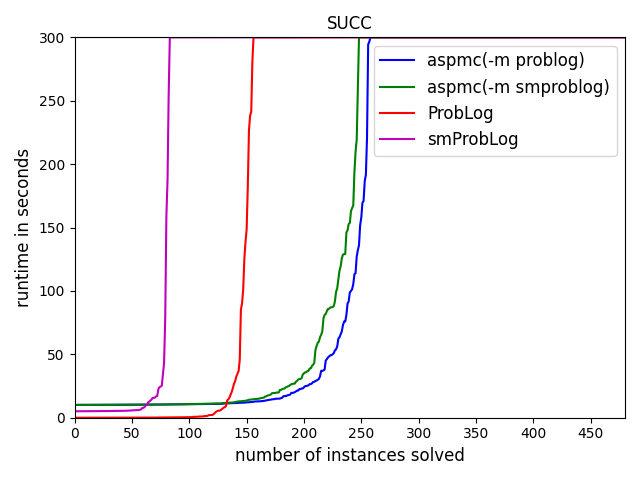}
    \caption{Q3: Running times of different solvers on SUCC programs.}
    \label{fig:meu_eff}
\end{figure}

\paragraph{\textbf{Q3}:   How does our second level approach compare to the first level approach when definedness reduces \twoAMC to AMC?}
On the regular smokers benchmark, \smp and \problog semantics coincide, i.e., all inner variables of the \twoAMC task are defined, and it thus reduces to AMC. In Figure~\ref{fig:meu_eff}, we plot running times for the AMC and \twoAMC variants of both \aspmc %, using c2d given the results for Q1, 
and \problog. For \problog, there is a clear gap between the two approaches, which is at least in part due to the fact that \problog only compiles the relevant part of the program, whereas \smp compiles the full theory. \aspmc outperforms \problog on the harder instances, with limited overhead for the second level task. 

%\begin{figure}% Q3
 %   \centering
 %   \includegraphics[width=0.5\textwidth]{efficiency_succ.png}
 %   \caption{Comparison of the runtimes of \aspmc, \problog on SUCC programs, once with SUCC\textsuperscript{sm} semantics (\aspmc(smproblog), \smp) and once with SUCC semantics (\aspmc(\problog), \problog). Runtimes are sorted ascendingly for each solver. }
  %  \label{fig:succ_eff}
%\end{figure}

%\\
%\textbf{START OLD VERSION, still to be updated}

%\input{experiments_long_part2}

%\textbf{END OLD VERSION}

\section{Conclusion}
\twoAMC is a hard problem, even harder than \#SAT or AMC in general, as it imposes significant constraints on variable orders in KC.
%since its solution usually requires that the outer variables are decided before the inner variables. 
Our theoretical results show that these constraints can be weakened by exploiting definedness of variables.
%Our theoretical results show this is actually not always the case: when variables are defined in terms of the outer variables, we can decide them freely at any point during compilation.
In practice, this allow us to
%Practically, we employ this result in a twofold manner. Firstly, 
(i)  introduce a strategy to construct variable orders for compilation into $\mathbf{X}/\mathbf{D}$-first sd-DNNFs with a-priori guarantees on complexity, and 
%that even gives a priori guarantees on how fast we can construct the circuit. Secondly,
(ii)  to use unit propagation to decide literals earlier than specified by the variable order during compilation. Our experimental evaluation shows that (ii) generally improves the performance and (i) can boost it when many variables are defined. Furthermore, we see that compilation usually performs much better than an enumeration based approach to solve \twoAMC. Last but not least, our extensions of \aspmc and \smp are competitive with PITA and \problog, the state of the art solvers for MAP, MEU and SUCC\textsuperscript{sm} inference for logic programs, and even exhibit improved performance on MEU and SUCC\textsuperscript{sm}.

\paragraph{\textbf{Acknowledgements}}
This work has been supported by the Austrian Science Fund (FWF) Grant W1255-N23 and by the Fonds Wetenschappelijk Onderzoek (FWO) project N. G066818N.

\bibliographystyle{tlplike}
\bibliography{bib.bib}

\begin{thebibliography}{}

\bibitem[Aziz et~al., 2015]{DBLP:conf/aaai/AzizCMS15}
{\sc Aziz, R.~A.}, {\sc Chu, G.}, {\sc Muise, C.~J.}, {\sc and} {\sc Stuckey,
  P.~J.}
\newblock Stable model counting and its application in probabilistic logic
  programming.
\newblock In {\em AAAI} 2015, pp. 3468--3474.

\bibitem[Barab{\'a}si and Bonabeau, 2003]{barabasi2003scale}
{\sc Barab{\'a}si, A.-L.} {\sc and} {\sc Bonabeau, E.} 2003.
\newblock Scale-free networks.
\newblock {\em {SCI AM}}, {\it 288}, 5, 60--69.

\bibitem[Baral et~al., 2009]{baral2009probabilistic}
{\sc Baral, C.}, {\sc Gelfond, M.}, {\sc and} {\sc Rushton, N.} 2009.
\newblock Probabilistic reasoning with answer sets.
\newblock {\em TPLP}, {\it 9}, 1, 57--144.

\bibitem[Bellodi et~al., 2020]{bellodi2020map}
{\sc Bellodi, E.}, {\sc Alberti, M.}, {\sc Riguzzi, F.}, {\sc and} {\sc Zese,
  R.} 2020.
\newblock {MAP} inference for probabilistic logic programming.
\newblock {\em TPLP}, {\it 20}, 5, 641--655.

\bibitem[Darwiche, 2004]{Darwiche04}
{\sc Darwiche, A.}
\newblock New advances in compiling {CNF} into decomposable negation normal
  form.
\newblock In {\em {ECAI}} 2004, pp. 328--332.

\bibitem[Darwiche, 2011]{darwiche2011sdd}
{\sc Darwiche, A.}
\newblock {SDD}: A new canonical representation of propositional knowledge
  bases.
\newblock In {\em IJCAI} 2011.

\bibitem[De~Raedt et~al., 2007]{de2007problog}
{\sc De~Raedt, L.}, {\sc Kimmig, A.}, {\sc and} {\sc Toivonen, H.}
\newblock Prob{L}og: A probabilistic {P}rolog and its application in link
  discovery.
\newblock In {\em IJCAI} 2007, volume~7, pp. 2462--2467.

\bibitem[Dell et~al., 2017]{Dell17}
{\sc Dell, H.}, {\sc Komusiewicz, C.}, {\sc Talmon, N.}, {\sc and} {\sc Weller,
  M.}
\newblock {PACE: The Second Iteration}.
\newblock In {\em International Symposium on Parameterized and Exact
  Computation (IPEC)} 2017, pp. 30:1---30:13.

\bibitem[Derkinderen and {De Raedt}, 2020]{derkinderen2020algebraic}
{\sc Derkinderen, V.} {\sc and} {\sc {De Raedt}, L.}
\newblock Algebraic circuits for decision theoretic inference and learning.
\newblock In {\em {ECAI}} 2020, volume 325, pp. 2569--2576.

\bibitem[Eiter et~al., 2021]{eiter2021treewidth}
{\sc Eiter, T.}, {\sc Hecher, M.}, {\sc and} {\sc Kiesel, R.}
\newblock Treewidth-aware cycle breaking for algebraic answer set counting.
\newblock In {\em {KR}} 2021, pp. 269--279.

\bibitem[Fierens et~al., 2015]{problog2system}
{\sc Fierens, D.}, {\sc den Broeck, G.~V.}, {\sc Renkens, J.}, {\sc Shterionov,
  D.~S.}, {\sc Gutmann, B.}, {\sc Thon, I.}, {\sc Janssens, G.}, {\sc and} {\sc
  {De Raedt}, L.} 2015.
\newblock Inference and learning in probabilistic logic programs using weighted
  boolean formulas.
\newblock {\em TPLP}, {\it 15}, 3, 358--401.

\bibitem[Gebser et~al., 2014]{gebser2014clingo}
{\sc Gebser, M.}, {\sc Kaminski, R.}, {\sc Kaufmann, B.}, {\sc and} {\sc
  Schaub, T.} 2014.
\newblock Clingo= {ASP+} control: Preliminary report.
\newblock {\em arXiv preprint arXiv:1405.3694},.

\bibitem[Kimmig et~al., 2011]{kimmig2011algebraic}
{\sc Kimmig, A.}, {\sc Van~den Broeck, G.}, {\sc and} {\sc De~Raedt, L.}
\newblock An algebraic prolog for reasoning about possible worlds.
\newblock In {\em AAAI} 2011.

\bibitem[Kimmig et~al., 2017]{kimmig2017algebraic}
{\sc Kimmig, A.}, {\sc {Van den Broeck}, G.}, {\sc and} {\sc {De Raedt}, L.}
  2017.
\newblock Algebraic model counting.
\newblock {\em JAL}, {\it 22}, 46--62.

\bibitem[Korhonen and J{\"{a}}rvisalo, 2021]{korhonen2021integrating}
{\sc Korhonen, T.} {\sc and} {\sc J{\"{a}}rvisalo, M.}
\newblock Integrating tree decompositions into decision heuristics of
  propositional model counters (short paper).
\newblock In {\em CP} 2021, volume 210 of {\em LIPIcs}, pp. 8:1--8:11.

\bibitem[Lagniez et~al., 2016]{lagniez2016improving}
{\sc Lagniez, J.-M.}, {\sc Lonca, E.}, {\sc and} {\sc Marquis, P.}
\newblock Improving model counting by leveraging definability.
\newblock In {\em IJCAI} 2016, pp. 751--757.

\bibitem[Latour et~al., 2017]{latour2017combining}
{\sc Latour, A. L.~D.}, {\sc Babaki, B.}, {\sc Dries, A.}, {\sc Kimmig, A.},
  {\sc {Van den Broeck}, G.}, {\sc and} {\sc Nijssen, S.}
\newblock Combining stochastic constraint optimization and probabilistic
  programming - from knowledge compilation to constraint solving.
\newblock In {\em CP} 2017, pp. 495--511.

\bibitem[Lee and Yang, 2017]{lee2017lpmln}
{\sc Lee, J.} {\sc and} {\sc Yang, Z.}
\newblock {LPMLN}, weak constraints, and {P-log}.
\newblock In {\em AAAI} 2017.

\bibitem[Manhaeve et~al., 2018]{manhaeve2018deepproblog}
{\sc Manhaeve, R.}, {\sc Dumancic, S.}, {\sc Kimmig, A.}, {\sc Demeester, T.},
  {\sc and} {\sc De~Raedt, L.}
\newblock Deep{P}rob{L}og: Neural probabilistic logic programming.
\newblock In {\em NeurIPS} 2018, pp. 3749--3759.

\bibitem[Oztok and Darwiche, 2015]{OztokDarwiche15}
{\sc Oztok, U.} {\sc and} {\sc Darwiche, A.}
\newblock A top-down compiler for sentential decision diagrams.
\newblock In {\em {IJCAI}} 2015, pp. 3141--3148.

\bibitem[Riguzzi and Swift, 2011]{riguzzi2011pita}
{\sc Riguzzi, F.} {\sc and} {\sc Swift, T.} 2011.
\newblock The pita system: Tabling and answer subsumption for reasoning under
  uncertainty.
\newblock {\em TPLP}, {\it 11}, 4-5, 433--449.

\bibitem[Skryagin et~al., 2021]{skryagin2021slash}
{\sc Skryagin, A.}, {\sc Stammer, W.}, {\sc Ochs, D.}, {\sc Dhami, D.~S.}, {\sc
  and} {\sc Kersting, K.} 2021.
\newblock {SLASH:} embracing probabilistic circuits into neural answer set
  programming.
\newblock {\em CoRR}, {\it abs/2110.03395}.

\bibitem[Totis et~al., 2021]{totis2021smproblog}
{\sc Totis, P.}, {\sc Kimmig, A.}, {\sc and} {\sc {De Raedt}, L.} 2021.
\newblock {SMProbLog}: Stable model semantics in {P}rob{L}og and its
  applications in argumentation.
\newblock {\em StarAI}, {\it abs/2110.01990}.

\bibitem[{Van den Broeck} et~al., 2010]{dtproblog}
{\sc {Van den Broeck}, G.}, {\sc Thon, I.}, {\sc van Otterlo, M.}, {\sc and}
  {\sc {De Raedt}, L.}
\newblock {DTProbLog}: {A} decision-theoretic probabilistic {P}rolog.
\newblock In {\em {AAAI}} 2010.

\bibitem[Yang et~al., 2020]{yang2020neurasp}
{\sc Yang, Z.}, {\sc Ishay, A.}, {\sc and} {\sc Lee, J.}
\newblock Neurasp: Embracing neural networks into answer set programming.
\newblock In {\em {IJCAI}} 2020, pp. 1755--1762.

\end{thebibliography}

\newpage
\appendix
\section{Full Proofs and Omitted Lemmas}
\label{sec:appendix}
\subsection{Full Proofs}
Here, we restate the Theorems and Lemmas from the main paper and give their full proofs.
\setcounter{theorem}{3}
\begin{theorem}[Tractable \twoAMC with $\mathbf{X}_O$-first sd-DNNFs]
Let $A = (\mathcal{T}, \mathbf{X}_I, \mathbf{X}_O, \alpha_I, \alpha_O, \mathcal{S}_I, \mathcal{S}_O, \transform)$ be a \twoAMC instance, where $\mathcal{T}$ is an $\mathbf{X}_O$-first sd-DNNF. Then, we can compute $\twoAMC(A)$ in linear time in the size of $\mathcal{T}$.
\end{theorem}
\begin{proof}
The basic idea is as follows: We can see $\mathcal{T}$ as an algebraic circuit, by replacing or-nodes, and-nodes, false and true by sum, product, zero and one, respectively and by replacing all literals by their weight. Then we only need to make sure to use the sum, product, zero and one from the correct semiring: For pure nodes $n$ such that $\mathit{Vars}(n) \subseteq X_I$ this is the inner one for all other nodes it is the outer one. Additionally, for mixed nodes $n = n_1 \srstimes^O n_2$, where w.l.o.g. $\mathit{Vars}(n_1) \subseteq X_I$  we need to use $\transform(n_1) \srstimes^O n_2$ to have values that are over the same semiring.

Consider a subgraph $n$ of $\mathcal{T}$ with exactly one outgoing edge for each or-node and all outgoing edges for each and-node. As $\mathcal{T}$ satisfies $\mathbf{X}_O$-first and is smooth, there is a node $n'$ in $n$ such that $\mathit{Vars}(n') = X_I$, i.e., exactly the outer variables occur above $n'$ (see also the lowest and-nodes of the left NNF in Figure~\ref{fig:sddnnfs}).
%all  outer variables and no inner variables occur above $n'$. 
Thus, $n'$ is equivalent to $\mathcal{T}\mid_{\mathbf{x}_O}$ for some assignment $\mathbf{x}_O$ to the outer variables, for which $n'$ computes the value of the inner AMC instance. 
As evaluation sums over all these subgraphs, it obtains the correct result.
\end{proof}

\setcounter{theorem}{8}
\begin{lemma}[Exponential Separation]
Let $\mathcal{T} = \bigwedge_{i = 1}^{n} X_i \leftrightarrow Y_i$, $\mathbf{X} = \{X_1, \dots, X_n\}$ and $\mathbf{D} = \{Y_1, \dots, Y_n\}$, then the size of the smallest $\mathbf{X}$-first sd-DNNF for $\mathcal{T}$ is exponential in $n$ and the size of the smallest $\mathbf{X}/\mathbf{D}$-first sd-DNNF for $\mathcal{T}$ is linear in $n$.
\end{lemma}
\begin{proof}
Let $\mathcal{C}$ be a $\mathbf{X}$-first sd-DNNF for $\mathcal{T}$. Recall that this implies that for every assignment $\mathbf{x}$ to $\mathbf{X}$ there exists a node $n_{\mathbf{x}}$ in $\mathcal{C}$ such that $n_{\mathbf{x}}$ is equivalent to $\mathcal{T}\mid_{\mathbf{x}}$. Since $\mathcal{T}\mid_{\mathbf{x}}$ and $\mathcal{T}\mid_{\mathbf{x}'}$ are different, whenever $\mathbf{x}$ and $\mathbf{x}'$ are different there are $2^{n}$ different residual theories $\mathcal{T}\mid_{\mathbf{x}}$, one for each $\mathbf{x} \in \inter(\mathbf{X})$. This implies that $\mathcal{C}$ has at least $2^{n}$ different nodes, which proves the first part of the lemma.

We still need to prove that there exists a $\mathbf{X}/\mathbf{D}$-first sd-DNNF $\mathcal{S}$ for $\mathcal{T}$ the size of which is linear in $n$. For this we first consider, what $\mathbf{X}/\mathbf{D}$-firstness means for $\mathcal{T}$. The variables that are defined by $\mathbf{X}$ in terms of $\mathcal{T}$ are $\mathbf{D} = \{Y_1, \dots, Y_n\}$. Thus, any node over variables $\mathbf{X}\cup\mathbf{D}$ is a pure node and every NNF over these variables is an $\mathbf{X}/\mathbf{D}$-first NNF. This means that we can use any sd-DNNF $\mathcal{S}$ for $\mathcal{T}$. It is easy to see, that when one decides variables according to the variable order $X_1, Y_1, \dots, X_n, Y_n$ the resulting NNF is of size linear in $n$.
\end{proof}

\setcounter{theorem}{11}
\begin{theorem}[Tractable AMC with $\mathbf{X}/\mathbf{D}$-first sd-DNNFs]
The value of a \twoAMC instance $A = (\mathcal{T}, \mathbf{X}_I, \mathbf{X}_O, \alpha_I, \alpha_O, \mathcal{S}_I, \mathcal{S}_O, \transform)$ can be computed in polynomial time under the following conditions:
\begin{itemize}
    \item $\mathcal{T}$ is an $\mathbf{X}/\mathbf{D}$-first sd-DNNF
    \item $\transform$ is a monoid homomorphism from the monoid  $\mathcal{M} = \langle O(A) \rangle_{(R_I, \srstimes^{I}, e_{\srstimes^{I}})}$ generated by the observable values to $(R_O, \srstimes^{O}, e_{\srstimes^{O}})$.
\end{itemize} 
\end{theorem}
% We prove this Theorem in two steps. First, we prove it for a stronger restriction:
% \begin{define}[$\mathbf{X}$-Firstness Modulo Global Definability]
% Given an NNF $n$ on variables partitioned into $\mathbf{X},\mathbf{Y}$, we say an internal node $n_i$ of $n$ is pure \emph{modulo global definability} if $\mathit{Vars}(n_i) \subseteq \mathbf{X} \cup \mathbf{D}(n, \mathbf{X})$ or $\mathit{Vars}(n_i) \subseteq \mathbf{Y}$ and mixed \emph{modulo definability}, otherwise. 
% $n$ is an $\mathbf{X}$-first NNF \emph{modulo global definability}, $\mathbf{X}/\mathbf{GD}$-first NNF for short, if for each of its and-nodes $n_i$ either all children of $n_i$ are pure modulo global definability, or one child of $n_i$ is mixed modulo global definability and all other children $n_j$ of $n_i$ are pure modulo global definability and $\mathit{Vars}(n_j) \subseteq \mathbf{X}\cup \mathbf{D}(n_i, \mathbf{X})$.
% \end{define}
% \begin{theorem}[Tractable AMC with $\mathbf{X}/\mathbf{GD}$-first sd-DNNFs]
% The value of a \twoAMC instance $A = (\mathcal{T}, \mathbf{X}_I, \mathbf{X}_O, \alpha_I, \alpha_O, \mathcal{S}_I, \mathcal{S}_O, \transform)$ can be computed in polynomial time under the following conditions:
% \begin{itemize}
%     \item $\mathcal{T}$ is an $\mathbf{X}/\mathbf{GD}$-first sd-DNNF
%     \item $\transform$ is a monoid homomorphism from the monoid  $\mathcal{M} = \langle O(A) \rangle_{(R_I, \srstimes^{I}, e_{\srstimes^{I}})}$ generated by the observable values to $(R_O, \srstimes^{O}, e_{\srstimes^{O}})$.
% \end{itemize} 
% \end{theorem}
\begin{proof}[Proof (Sketch).]
% Let $v$ be an $\mathbf{X}_O/\mathbf{D}$-constrained node of the vtree of $T$.
% We define $\mathbf{D}^*$ as the subset variables in $\mathbf{D}$ that do not occur in the subtree rooted at $v$.
Let $n$ be the root of $\mathcal{T}$.
We know that the variables in $\mathbf{D}(n, \mathbf{X}_O)$ are defined by $\mathbf{X}_O$ in terms of $T$. For $d \in \mathbf{D}(n, \mathbf{X}_O)$ and $\mathbf{x}_O \in \inter(\mathbf{X}_O)$, we denote by $d\mid_{\mathbf{x}_O}$ the literal of $d$ that must be included in $\mathbf{x}_I$ in order for $\mathbf{x}_I \cup \mathbf{x}_O$ to be a satisfying assignment (if $\mathbf{x}_O$ can be extended to a satisfying assignment, otherwise choose an arbitrary but fixed value).

Recall that the value of $A$ is defined as
\begin{align*}
    \srbplus^{O}_{\mathbf{x}_O \in \inter(\mathbf{X}_O)} \srbtimes^{O}_{x \in \mathbf{x}_O} \alpha_O(x) \srstimes^{O} \transform\left(\srbplus^{I}_{\mathbf{x}_I \in \inter(\mathbf{X}_I), \mathbf{x}_I\cup\mathbf{x}_O \models T} \srbtimes^{I}_{y \in \mathbf{x}_I} \alpha_I(y)\right).
\end{align*}
Since the inner sum only takes interpretations that satisfy $T$, we do not need to take the sum over both values of a defined variable $d$ but can restrict ourselves to the value $d\mid_{\mathbf{x}_O}$ determined by $\mathbf{x}_O$.
\begin{align*}
    &\srbplus^{I}_{\mathbf{x}_I \in \inter(\mathbf{X}_I), \mathbf{x}_I\cup\mathbf{x}_O \models T} \srbtimes^{I}_{y \in \mathbf{x}_I} \alpha_I(y) \\
    = &\srbplus^{I}_{\mathbf{x}_I \in \inter(\mathbf{X}_I \setminus \{d\}), \mathbf{x}_I\cup\mathbf{x}_O \models T} \srstimes^{I} \alpha_I(d\mid_{\mathbf{x}_O}) \srstimes^{I} \srbtimes^{I}_{y \in \mathbf{x}_I} \alpha_I(y) \\
    = &\alpha_I(d\mid_{\mathbf{x}_O}) \srstimes^{I} \srbplus^{I}_{\mathbf{x}_I \in \inter(\mathbf{X}_I \setminus \{d\}), \mathbf{x}_I\cup\mathbf{x}_O \models T} \srbtimes^{I}_{y \in \mathbf{x}_I} \alpha_I(y) 
\end{align*}
Next, we can plug these equalities into the expression for the value of $A$ and use that $\transform$ is a homomorphism.
\begin{align*}
    &\srbplus^{O}_{\mathbf{x}_O \in \inter(\mathbf{X}_O)} \srbtimes^{O}_{x \in \mathbf{x}_O} \alpha_O(x) \srstimes^{O} \transform\left(\srbplus^{I}_{\mathbf{x}_I \in \inter(\mathbf{X}_I), \mathbf{x}_I\cup\mathbf{x}_O \models T} \srbtimes^{I}_{y \in \mathbf{x}_I} \alpha_I(y)\right)\\
    =&\srbplus^{O}_{\mathbf{x}_O \in \inter(\mathbf{X}_O)} \srbtimes^{O}_{x \in \mathbf{x}_O} \alpha_O(x) \srstimes^{O} \transform\left(\alpha_I(d\mid_{\mathbf{x}_O}) \srstimes^{I} \srbplus^{I}_{\mathbf{x}_I \in \inter(\mathbf{X}_I \setminus \{d\}), \mathbf{x}_I\cup\mathbf{x}_O \models T} \srbtimes^{I}_{y \in \mathbf{x}_I} \alpha_I(y) \right)\\
    =&\srbplus^{O}_{\mathbf{x}_O \in \inter(\mathbf{X}_O)} \srbtimes^{O}_{x \in \mathbf{x}_O} \alpha_O(x) \srstimes^{O} \transform( \alpha_I(d\mid_{\mathbf{x}_O})) \srstimes^{O} \transform\left(\srbplus^{I}_{\mathbf{x}_I \in \inter(\mathbf{X}_I \setminus \{d\}), \mathbf{x}_I\cup\mathbf{x}_O \models T} \srbtimes^{I}_{y \in \mathbf{x}_I} \alpha_I(y) \right)
\end{align*}
Due to fact that $\transform$ satisfies $\transform(e_{\srstimes^{I}}) = e_{\srstimes^{O}}, \transform(e_{\srsplus^{I}}) = e_{\srsplus^{O}}$, whenever the assignment $\mathbf{x}_O$ cannot be extended to a satisfying of $T$, the weight for the given assignment will be $e_{\srsplus^{O}}$. 

Thus, we can again use the fact that the variable $d$ is defined and sum over both of its values \emph{in the outer sum}, resulting in
\begin{align*}
    &\srbplus^{O}_{(\mathbf{x}_O,\mathbf{d}) \in \inter(\mathbf{X}_O \cup \{d\}} \srbtimes^{O}_{x \in \mathbf{x}_O} \alpha_O(x) \srstimes^{O} \transform\left(\srbtimes^{I}_{d \in \mathbf{d}} \alpha_I(d)\right) \srstimes^{O} \transform\left(\srbplus^{I}_{\mathbf{x}_I \in \inter(\mathbf{X}_I \setminus \{d\}), \mathbf{x}_I\cup\mathbf{x}_O \models T} \srbtimes^{I}_{y \in \mathbf{x}_I} \alpha_I(y) \right)\\
    = &\srbplus^{O}_{(\mathbf{x}_O,\mathbf{d}) \in \inter(\mathbf{X}_O \cup \{d\})} \srbtimes^{O}_{x \in \mathbf{x}_O} \alpha_O(x) \srstimes^{O} \srbtimes^{O}_{d \in \mathbf{d}} \transform(\alpha_I(d)) \srstimes^{O} \transform\left(\srbplus^{I}_{\mathbf{x}_I \in \inter(\mathbf{X}_I \setminus \{d\}), \mathbf{x}_I\cup\mathbf{x}_O \models T} \srbtimes^{I}_{y \in \mathbf{x}_I} \alpha_I(y) \right)
\end{align*}
We observe that this expression is equal to the value of another \twoAMC instance $B = (T, \mathbf{X}_I\setminus \{d\}, \mathbf{X}_O\cup \{d\}, \beta_{I},  \beta_{O}, \mathcal{R}_I, \mathcal{R}_O)$, where 
\begin{align*}
    % \beta_{I}(x) &= \alpha_I(x) & \text{ for } x \in \lit(\mathbf{X}_I\setminus \{d\})\\
    % I like more this alignment
    \beta_{I}(x) &= \alpha_I(x) \quad \text{ for } x \in \lit(\mathbf{X}_I\setminus \{d\})\\
    \beta_{O}(x) &= \left\{\begin{array}{cc}
        \alpha_O(x) & \text{ if } x \in \lit(\mathbf{X}_O), \\
        \transform(\alpha_I(x)) & \text{ if } x \in \lit(\{d\}).
    \end{array}\right..
\end{align*}
Now, if $d$ occurs before some variable $x \in \mathbf{X}$ in some node $n$ of $\mathcal{T}$, then $n$ is not an $\mathbf{X}$-first NNF. However, it is an $\mathbf{X}_O\cup \{d\}$-first sd-DNNF. On this NNF we can solve the \twoAMC-instance $B$ in polynomial time according to Theorem~\ref{thm:tract1}.
By induction on the number of variables that occur before $x$ the claim follows.
\end{proof}

\setcounter{theorem}{13}
\begin{lemma}
Let $\mathcal{T}$ be a CNF over variables $\mathbf{Y}$ and $(T,\chi)$ a TD of $\PRIM(\mathcal{T})$ of width $k$. Furthermore, let $\mathbf{X} \subseteq \mathbf{Y}$ and $\mathbf{D} = \mathbf{D}(\mathcal{T}, \mathbf{X})$. If there exists $t^* \in V(T)$ such that (1) $\chi(t^*) \subseteq \mathbf{X} \cup \mathbf{D}$ and (2) $\chi(t^*)$ is a \emph{separator} of $\mathbf{X}$ and $\mathbf{Y}\setminus (\mathbf{X}\cup\mathbf{D})$, i.e., every path from $\mathbf{X}$ to $\mathbf{Y}\setminus (\mathbf{X}\cup\mathbf{D})$ in $\PRIM(\mathcal{T})$ uses a vertex from $\chi(t^*)$, then we can compile $\mathcal{T}$ into an $\mathbf{X}/\mathbf{D}$-first sd-DNNF in time $\mathcal{O}(2^{k}\cdot \text{poly}(|\mathcal{T}|))$. 
\end{lemma}
\begin{proof}
The performance guarantee is due to~\cite{korhonen2021integrating} and holds when we decide the variables in the order they occur in the TD starting from the root.
$\mathbf{X}/\mathbf{D}$-firstness can be guaranteed by taking $t^*$ as the root of the TD and, thus, first deciding all variables in $\chi(t^*) = \{S_1, \dots, S_n\}$. 
From condition (2) it follows that afterwards the CNF has decomposed into separate components, which either only use variables from $\mathbf{X} \cup \mathbf{D}$ or use no variables from $\mathbf{X}$. Thus, their compilation only leads to pure NNFs.

The variable order can be inspected schematically in Figure~\ref{fig:xdvtree}. Here, $v_{I}$ is the remaining variable order for the inner variables and $v_{O}$ is the remaining variable order for the outer (and possibly defined) variables. The split signifies that the CNF composes into different components and we can consider respective variable orders for them independently. 
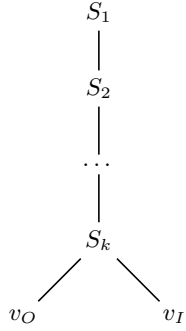
\begin{figure}
    \centering
    \begin{tikzpicture}
        \node (v0) at (0, 1) {$S_{1}$};
        \node (v1) at (0, 0) {$S_{2}$};
        \node (v2) at (0, -1) {$\dots$};
        \node (v3) at (0, -2) {$S_{k}$};
        \node (v4) at (-1, -3) {$v_{O}$};
        \node (v5) at (1, -3) {$v_{I}$};
        \draw[-,semithick] (v0) to (v1);
        \draw[-,semithick] (v1) to (v2);
        \draw[-,semithick] (v2) to (v3);
        \draw[-,semithick] (v3) to (v4);
        \draw[-,semithick] (v3) to (v5);
    \end{tikzpicture}
    \caption{Schematic variable order constructed in the proof of Lemma~\ref{lem:XD-tree}.}
    \label{fig:xdvtree}
\end{figure}
\end{proof}

\subsection{Omitted Lemmas Showing Homomorphism Property}
Here, we give proofs for the fact that the transformation functions of the different \twoAMC problems are homomorphisms.
\renewcommand\thetheorem{\roman{theorem}}
\setcounter{theorem}{0}
\begin{lemma}
The function
\[
    h: \mathbb{N}^2 \rightarrow [0,1], (n_1, n_2) \mapsto \left\{\begin{array}{cc}
        0 & \text{ if } n_2 = 0,\\
        \nicefrac{n_1}{n_2} & \text{ otherwise.}
    \end{array}\right. 
\] 
is a monoid homomorphism from $(\{(n_1, n_2) \mid n_1 \leq n_2, n_1, n_2 \in \mathbb{N}\}, \cdot, (1,1))$, where multiplication is coordinatewise, to $([0,1], \cdot, 1)$. 
\end{lemma}
Note that if $n_2$ is zero then also $n_1$ is zero, so no other division by zero is avoided by $h$.
\begin{proof}
It is clear that $h((1,1)) = 1$, therefore the neutral element is preserved. Furthermore, let $(n_1, n_2), (n_3, n_4) \in \mathbb{N}^2$. First assume, that one of $n_1$ and $n_3$ is zero. In this case, 
\[
    h((n_1, n_2) \cdot (n_3, n_4)) = h((0, n_2n_4)) = 0 = h((n_1,n_2))h((n_3,n_4)), 
\]
since one of $h((n_1,n_2)$ and $h((n_3,n_4))$ is zero. Otherwise, both $n_1$ and $n_3$ and therefore also $n_2$ and $n_4$ are unequal to zero.
Then
\[
    h((n_1, n_2) \cdot (n_3, n_4)) = h((n_1n_3, n_2n_4)) = \frac{n_1n_3}{n_2n_4} = \frac{n_1}{n_2}\frac{n_3}{n_4} = h((n_1,n_2))h((n_3,n_4)).
\]
\end{proof}

\begin{lemma}
Let $A$ be a \twoAMC instance corresponding to the evaluation problem of an \dtp program. Then 
\[
    h : \{(p, pu) \mid p \in [0,1], u \in \mathbb{R}\} \rightarrow \mathbb{R}\cup\{-\infty\}, (p, pu) \mapsto \left\{\begin{array}{cc}
        -\infty & \text{ if } p = 0 \\
        pu & \text{ otherwise.}
    \end{array}\right..
\]
is a monoid homomorphism from the monoid generated by the observable values to $(\mathbb{R}\cup\{-\infty\}, +, 0)$.
\end{lemma}
While $h$ is not a homomorphism on the whole set of values, it is one on the set of observable values, since they are all of the form $(p, pu)$ for $p \in \{0,1\}$.
\begin{proof}
For \problog programs it is known, that probabilistic facts are not defined, instead, every assignment of the probabilistic facts can be uniquely extended to a satisfying assignment of the program. 

This implies two things that are useful for us. Firstly, the probability of a \problog program being satisfied is always $1$. Since an \dtp program is a \problog program conditioned on any complete assignment to the decision variables, this also implies that the expected utility of the conditioned program is of the form $(1, u)$. This means that all the values of the inner sum are of the form $(1, u)$ or $(0,0)$.

Secondly, since probabilistic facts cannot be defined, any defined variable must have weights of the form $(1, u)$ as well, which also implies that even if we take out the defined variables from the inner sum, the value of the inner sum is still of the form $(1,u)$ or $(0,0)$.

Overall, it follows that the monoid generated by the observable values is as submonoid of
\[
    \mathcal{M} = (\{(1, u) \mid u \in \mathbb{R}\} \cup (0,0), \otimes, (1,1)).
\]
Now, the only thing left to prove is that $h$ is a monoid homomorphism from $\mathcal{M}$ to $(\mathbb{R}\cup\{-\infty\}, +, 0)$. This can again be easily as follows.

Let $(p, pu), (p', p'u') \in \mathcal{M}$. If $p$ or $p'$ are zero, we have
\[
    h((p, pu)\otimes (p', p'u')) = h((0,0)) = -\infty = h((p, pu))h((p', p'u')).
\]
Otherwise, we know that both are one. Then
\[
    h((1, u)\otimes (1, u')) = h((1,u+u')) = u + u' = h((1, u)) + h((1, u')).
\]
\end{proof}

\end{document}